\documentclass{article}
\usepackage{ijcai21}

\usepackage{times}

\usepackage{soul}
\usepackage{url}
\usepackage[hidelinks]{hyperref}
\usepackage[utf8]{inputenc}
\usepackage[small]{caption}
\usepackage{graphicx}
\usepackage{amsmath}
\usepackage{booktabs}
\usepackage{stfloats} 
\usepackage{enumitem}
\usepackage{amsmath}
\usepackage{amssymb}
\usepackage{bm}
\usepackage{booktabs}
\usepackage{multirow}
\usepackage{graphicx}
\usepackage{subfigure}
\usepackage{algorithm}
\usepackage{algorithmic}
\usepackage{amsthm}
\usepackage[american]{babel}

\newcommand\ie{\emph{i.e.}}
\newcommand\eg{\emph{e.g.}}

\newtheorem{theorem}{Theorem}

\newtheorem{lemma}{Lemma}
\usepackage[title]{appendix}

\title{Towards Understanding Deep Learning from Noisy Labels\\ with Small-Loss Criterion}

\author{
Xian-Jin Gui\and
Wei Wang\footnote{Corresponding author.}\And
Zhang-Hao Tian\\
\affiliations
National Key Laboratory for Novel Software Technology\\
Nanjing University, Nanjing 210023, China\\
\emails
\{guixj, wangw, tianzh\}@lamda.nju.edu.cn
}

\begin{document}
\maketitle

\begin{abstract}
Deep neural networks need large amounts of labeled data to achieve good performance. In real-world applications, labels are usually collected from non-experts such as crowdsourcing to save cost and thus are noisy. In the past few years, deep learning methods for dealing with noisy labels have been developed, many of which are based on the small-loss criterion. However, there are few theoretical analyses to explain why these methods could learn well from noisy labels. In this paper, we theoretically explain why the widely-used small-loss criterion works. Based on the explanation, we reformalize the vanilla small-loss criterion to better tackle noisy labels.
The experimental results verify our theoretical explanation and also demonstrate the effectiveness of the reformalization.
\end{abstract}

\section{Introduction}
Deep neural networks (DNNs) have achieved great success in many real-world applications, but rely on large-scale data with accurate labels~\cite{deng2009imagenet}. Obtaining large-scale accurate labels is expensive while the alternative methods such as crowdsourcing~\cite{raykar2010learning} and web queries~\cite{jiang2020beyond} can easily provide extensive labeled data, but unavoidably incur noisy labels. The performance of deep neural networks may be severely hurt if these noisy labels are blindly used~\cite{zhang2017understanding}, and thus how to learn with noisy labels has become a hot topic.

In the past few years, many deep learning methods for tackling noisy labels have been developed. Some methods try to exploit noise-robust loss functions, \eg, MAE loss~\cite{ghosh2017robust}, Truncated $\mathcal{L}_q$ loss~\cite{zhang2018generalized} and the information-theoretic loss~\cite{xu2019l_dmi}. These methods do not consider the specific information about label noise, and thus usually have limited utility in real-world applications. Some methods use the transition matrix to model label noise and construct an unbiased loss term to alleviate the influence of noisy labels~\cite{sukhbaatar2014training,patrini2017making,goldberger2017training,han2018masking,hendrycks2018using}. However, the performance of these methods is usually suboptimal due to the difficulty of accurately estimating the noise transition matrix. Some other methods try to correct the noisy labels~\cite{ma2018dimensionality,arazo2019unsupervised,yi2019probabilistic}, but may suffer from the false correction. Sometimes, although correcting the noisy labels might be challenging especially for the classification task with a large number of classes, the detection of noisy labels is relatively easy. Along this direction, the sample selection strategy with the widely-used small-loss criterion has been proposed, \ie, treating the examples with small loss as the clean data and using them in the training process. Although many methods based on the small-loss criterion have achieved prominent performance in practice~\cite{han2018co,yu2019does,shen2019learning,song2019selfie,wei2020combating}, the theoretical explanation about when and why it works is rarely studied. 

When there are noisy labels in the data, it is somehow overly optimistic to expect that deep neural networks could achieve good performance without any assumption on label noise. Thus, most of previous studies potentially make assumptions on label noise, \eg, the condition that correct labels are not overwhelmed by the false ones~\cite{sukhbaatar2014training,han2018co}. Some methods focus on the class-conditional noise setting~\cite{natarajan2013learning,patrini2017making}, \ie, the label noise class-conditionally depends only on the latent true class, but not on the feature. This assumption is an approximation of real-world label noise and can encode the similarity information between classes. 
Based on this, three representative types of label noise have been considered, \ie, uniform label noise~\cite{hendrycks2018using}, pairwise label noise~\cite{han2018co} and structured label noise~\cite{patrini2017making,zhang2018generalized}. For these types of label noise, it is usually assumed that the diagonally-dominant condition holds, and many methods could achieve good performance with this condition~\cite{rolnick2017deep,wei2020combating}. Unfortunately, there are few theoretical analyses to explain why this diagonally-dominant condition is necessary for good performance. 
In this work, we first reveal the theoretical condition under which learning methods could achieve good performance with noisy labels, which exactly matches the condition assumed in previous methods, 
and then theoretically explain when and why the small-loss criterion works. Based on the explanation, we reformalize the vanilla small-loss criterion to better tackle noisy labels.
The experimental results on synthetic and real-world datasets verify our theoretical results and demonstrate the effectiveness of the reformalization of small-loss criterion.

\section{Related Work} 
There are many existing methods for learning from noisy labels~\cite{algan2019image}, we only briefly introduce the most related ones herein. 

For tackling noisy labels, the methods based on robust loss functions have been proposed. \cite{ghosh2017robust} proved Mean Absolute Error (MAE) is more robust to label noise than Cross-Entropy Error (CCE), but MAE is hard to optimize due to gradient saturation issues. Later, \cite{zhang2018generalized} constructed Truncated $\mathcal{L}_q$ loss by combining MAE and CCE. 
Recently, \cite{xu2019l_dmi} proposed the information-theoretic loss $\mathcal{L}_{\text{DMI}}$ to tackle noisy labels. Although these methods could alleviate the influence of label noise to some extent, they do not take the information of label noise into consideration. Some methods try to first estimate the noise transition matrix and then use it to correct the loss term~\cite{patrini2017making,hendrycks2018using}, while a few others add a noise layer into the network to implicitly simulate the noise transition  
process~\cite{sukhbaatar2014training,goldberger2017training}. When there exists noise in the labels, it is beneficial if we could correct some labels. 
\cite{tanaka2018joint} and \cite{yi2019probabilistic} tried to correct noisy labels by jointly optimizing the model parameters and noisy labels, while~\cite{ma2018dimensionality} and~\cite{arazo2019unsupervised} used the convex combination of noisy labels and model's predictions as training targets to reduce noisy labels' influence. Obviously, some false corrections may be incurred in the label correction process. To exempt from false correction, some methods try to select a part of low-risk data based on the small-loss criterion. 
Co-teaching~\cite{han2018co} trains two networks simultaneously and update each network on the data selected by the other with the small-loss criterion. Later, Co-teaching+~\cite{yu2019does} improves Co-teaching by maintaining the divergence between the two networks. INCV~\cite{chen2019understanding} first splits the noisy dataset into a selected set, a candidate set and a removed set based on validation loss, then exploits Co-teaching strategy to learn. \cite{wei2020combating} further claimed that the agreement is important in the learning process and proposed the method JoCoR which combines Co-teaching with Co-regularization.

Although these methods have achieved prominent performance, there are only few studies to explain the underlying mechanism of deep learning with noisy labels. 
Some work tries to demystify the intriguing memorization phenomenon~\cite{arpit2017a}, \ie, DNNs first fit correct labels and then start to overfit incorrect labels.
\cite{li2020gradient} theoretically proved that under a cluster assumption on data for least-square regression tasks, a one-hidden layer neural network trained with gradient descent will exhibit this phenomenon. \cite{liu2020early} provided a theoretical characterization of memorization phenomenon for a high-dimensional linear generative model under binary classification task. 
Some work focuses on how regularization helps deep learning with noisy labels. \cite{hu2020simple} proposed two simple regularization methods and proved that gradient descent training with either of them by using noisy labels enjoys a good generalization guarantee.

\section{Preliminaries}
We focus on the classification task in this paper. Let $\mathcal{X}$ denote the instance space, for each $\bm x\in\mathcal{X}$, there exists a true label $y\in\mathcal{Y}=\{1,\dots,c\}$ determined by the target concept $f^{*}$, \ie, $y=f^{*}(\bm x)$, and $c$ is the number of classes. In real-world applications, for an instance $\bm x$, the observed label $\tilde{y}$  may be corrupted. 
In previous studies, the class-conditional noise assumption (\ie, $p(\tilde{y}|y, \bm x) = p(\tilde{y}|y)$) is popularly used~\cite{ghosh2017robust,patrini2017making,sukhbaatar2014training,xia2019are,xu2019l_dmi}. 
The label corruption can be described by a noise transition matrix $T\in \mathbb{R}^{c\times c}$, where $T_{ij} = p(\tilde{y}=j|y=i)$ denotes the probability of an $i$-th class example flipped into the $j$-th class, and the noisy data distribution satisfies $p(\bm x,\tilde y)=\sum_{i=1}^cp(\tilde{y}|y=i)p(\bm x,y=i)$.

Usually, we have a training dataset $\tilde{D}=\{(\bm x_i, \tilde{y}_i)\}_{i=1}^n$ with noisy labels. We consider the deep neural network $g(\bm x; \Theta):\mathcal{X} \to \mathbb{R}^c$ with output $g(\bm x;\Theta) = [\hat{p}_1(\bm x), \dots, \hat{p}_c(\bm x)]^\top\in\mathbb{R}^c$, \[\hat{p}_i(\bm x) = \frac{\exp\big(\bm{w}_i^\top \phi(\bm x;\bm{\theta})\big)}{\sum_{j=1}^c\exp\big(\bm{w}_{j}^\top \phi(\bm x;\bm{\theta})\big)},\] where $\bm{w}_i$ is the weight of the softmax classifier for the $i$-th class (the bias term is omitted for brevity), $\phi(\bm x;\bm{\theta})$ denotes the output of the penultimate layer of the deep neural network, and $\Theta = vec(\bm\theta, \{\bm{w}_i\}_{i=1}^{c})$ denotes the vectorization of all model parameters. In essence, $\phi(\cdot;\bm{\theta})$ can be regarded as a feature extractor which yields a new representation $\phi(\bm x;\bm{\theta})$ for each input $\bm x$. Sometimes, we omit the parameter $\Theta$ and denote $g(\bm x;\Theta)$ as $g(\bm x)$ for brevity. The classifier induced by $g$ is $f_g(\bm x)  = \mathop{\arg\max}_{i\in \{1,\dots,c\}}\, \hat{p}_i(\bm x) \in \mathcal{Y}$. For an example $(\bm x, \tilde{y})$, the loss is calculated as $\ell(g(\bm x), \tilde{y})$ with a given loss function $\ell(\cdot, \cdot)$. The empirical loss of $g$ on $\tilde{D}$ is $\frac{1}{n}\sum_{i=1}^n\ell(g(\bm x_i), \tilde{y}_i)$ and the expected loss on noisy data is $\mathbb E_{(\bm x,\tilde{y})}[\ell(g(\bm x), \tilde{y})]$.

\section{Our Work}

The goal of learning from noisy labels is usually to obtain a classifier that performs as well as possible on a clean test set, and even it is expected to learn the target concept $f^*$.
It is overly optimistic to expect to learn a good model without any assumption on the noise transition matrix $T$ since it characterizes the distribution of label noise. Many previous methods made assumptions on $T$, 
\eg, $T$ is a diagonally dominant matrix~\cite{han2018co,yu2019does}. 
This assumption originally comes from the empirical study, and why it is necessary for achieving good performance is still unclear.
Considering the classification function $f:\mathcal{X}\to\mathcal{Y}$, the 0-1 loss on $(\bm x, \tilde{y})$ is $\ell_{01}(f(\bm x), \tilde{y}) = \mathbb I [f(\bm x)\neq \tilde{y}]$. The empirical loss of $f$ on the noisy dataset $\tilde{D}$ is $\frac{1}{n}\sum_{i=1}^n\ell_{01}(f(\bm x_i), \tilde{y}_i)$ and the expected loss is $\mathbb E_{(\bm x,\tilde{y})}[\ell_{01}(f(\bm x), \tilde{y})]$. 
In the following part, we theoretically explain why the diagonally-dominant condition is important in the learning process.\footnote{Due to space limit, all proofs of the lemmas and theorems can be found in Appendix~A.}
\begin{lemma}\label{lemma:1988}
If $T$ satisfies the row-diagonally dominant condition $T_{ii}>\max_{j\neq i}T_{ij}$, $\forall i$, then the target concept $f^*$ has the minimum expected 0-1 loss on the noisy data, i.e., $\forall \, f\neq f^*$, $\mathbb E_{(\bm x,\tilde{y})}[\ell_{01}(f^*(\bm x), \tilde{y})] \le \mathbb E_{(\bm x,\tilde{y})}[\ell_{01}(f(\bm x), \tilde{y})]$.
\end{lemma}
Lemma~\ref{lemma:1988} indicates that when the row-diagonally dominant condition is met, the target concept classifier $f^*$ has the minimum expected 0-1 loss on the noisy data, and can be learned with Empirical Risk Minimization (ERM) methods. 
Since 0-1 loss is difficult to optimize, the convex surrogate loss is usually used in practice, \eg, the cross-entropy loss $\ell_{CE}(g(\bm x;\Theta), \tilde{y}) = -\log(\hat{p}_{\tilde y}(\bm x))$. 
For deep neural networks with cross-entropy loss,
the learning process is to find the model $g^* = g(\bm x;\Theta^*)$
minimizing the expected loss, \ie,
\begin{equation}\label{eq:optimal-para}
\Theta^*=\mathop{\arg\min}_{\Theta}\mathbb E_{(\bm x,\tilde{y})}[\ell_{CE}(g(\bm x;\Theta), \tilde{y})].
\end{equation}

\begin{lemma}\label{lemma:row-dominant}
Let $g^*$ denote the deep neural network minimizing the cross-entropy loss in Eq.~\eqref{eq:optimal-para}, the induced classifier $f_{g^*}$ satisfies $f_{g^*}(\bm x)=y$, $\forall \bm x\in\mathcal{X}$, if and only if $T$ satisfies the row-diagonally dominant condition $T_{ii}> \max_{j\neq i}T_{ij}$, $\forall i$.
\end{lemma}
Lemma~\ref{lemma:row-dominant} indicates the condition for learning the optimal classifier from noisy data.
This result matches the assumption potentially made by previous methods and explains why they could achieve good performance, \eg, \cite{rolnick2017deep} empirically showed that deep neural networks are immune to some kinds of label noise. 

Although Lemmas~\ref{lemma:1988} and~\ref{lemma:row-dominant} theoretically show that good classifiers could be learned from noisy data when some certain condition is met, it is difficult to obtain them by ERM methods since in practice we only have a given noisy dataset with finite examples and deep neural networks can even memorize these finite noisy examples due to over-parameterization. In real-world applications, many methods resort to sample selection strategy with the small-loss criterion to deal with noisy labels. This  criterion can be summarized as the following process: for a warmed-up neural network $g$, it first selects the examples with small loss and then update the model parameter with these small-loss examples~\cite{han2018co,yu2019does,song2019selfie,wei2020combating}.
However, the theoretical explanation about why this small-loss criterion works has not been touched. Now we provide a theoretical explanation for this.

\begin{theorem}\label{thm:the-small-loss}
Let $g^*$ denote the deep neural network minimizing the cross-entropy loss in Eq.~\eqref{eq:optimal-para}, $(\bm x_1,\tilde{y})$ and $(\bm x_2,\tilde{y})$ are any two examples with the same observed label $\tilde{y}$ in $\tilde{D}$ satisfying that $f^*(\bm x_1)=\tilde{y}$ and $f^*(\bm x_2)\neq \tilde{y}$, 
if $T$ satisfies the diagonally-dominant condition $T_{ii} > \max\,\{\max_{j \neq i}T_{ij}, \,\,\max_{j\neq i}T_{ji}\}$, $\forall i$, then $\ell_{CE}(g^*(\bm x_1), \tilde{y})<\ell_{CE}(g^*(\bm x_2), \tilde{y})$. 

\end{theorem}

Theorem~\ref{thm:the-small-loss} indicates that when the noise transition matrix $T$ satisfies the diagonally-dominant condition, for the neural network $g^*$, considering the examples with the same observed label, the correct examples have smaller loss than the incorrect ones. Here $g^*$ is the neural network minimizing the expected cross-entropy loss on the noisy data. While  in practice, we may only obtain a neural network $g$ which is trained relatively well enough but not necessary to be $g^*$.
Suppose $g$ is $\epsilon$-close to $g^*$, 
\ie, $\|g-g^*\|_\infty := \sup_{\bm x\in\mathcal{X}}\|g(\bm x)-g^*(\bm x)\|_\infty = \epsilon$, we analyze whether the small-loss criterion still can hold.

\begin{theorem}\label{thm:weak-small-loss}
Suppose $g$ is $\epsilon$-close to $g^*$, i.e., $\|g-g^*\|_\infty = \epsilon$, for two examples $(\bm x_1, \tilde{y})$ and $(\bm x_2, \tilde{y})$, assume $f^*(\bm x_1)=\tilde{y}$ and $f^*(\bm x_2)\neq \tilde{y}$, if $T$ satisfies the diagonally-dominant condition $T_{ii} > \max\,\{\max_{j \neq i}T_{ij}, \,\,\max_{j\neq i}T_{ji}\}$, $\forall i$, and $\epsilon < \frac{1}{2}\cdot(T_{\tilde{y}\tilde{y}}-T_{f^*(\bm x_2)\tilde{y}})$, then $\ell_{CE}(g(\bm x_1), \tilde{y}) < \ell_{CE}(g(\bm x_2), \tilde{y})$. 
\end{theorem}

Theorem~\ref{thm:weak-small-loss} indicates that if $g$ is not far away from $g^*$, for examples with the same observed label, the correct examples still have smaller loss than the incorrect ones. 
In practice, when trained with finite examples,
after a warm-up stage for model $g$, 
the condition in Theorem~\ref{thm:weak-small-loss} may not hold for all $(\bm{x}_i, \tilde{y}_i)\in\tilde{D}$.
It usually holds for a part of examples and we can select a small part of clean data with the small-loss criterion.
This provides an explanation for the effectiveness of existing methods with the small-loss criterion, since the noise setting that they considered satisfies the diagonally-dominant condition. We will show by experiments when this condition does not hold, the small-loss criterion may fail.

With the above discussions, the small-loss criterion can be used to separate correct examples from incorrect ones. 
While according to the theoretical explanation, it would be better to select examples class by class by comparing the loss of examples with the same observed labels.
Actually, we will show by experiments that the loss of examples with different noisy labels may be not comparable.
Besides, the loss of examples may fluctuate in different epochs since in practice DNNs are optimized by stochastic gradient descent.
Thus we propose to reformalize the small-loss criterion as follow: we use the \emph{mean loss} of each example along the training process and select the examples with small mean loss \emph{class by class}.
The overall process of the proposed reformalization of small-loss criterion (\textbf{RSL}) is shown in Algorithm~\ref{alg:selection}. 

The selected number of examples in Algorithm~\ref{alg:selection} is an important parameter that needs to be set carefully. For the $i$-th class with noise rate $\eta_i$, it is reasonable that the selected proportion $\textit{prop}(i)$ is a little less than $1-\eta_i$. Considering $\eta_i$ may be larger than $0.5$, we propose to set $\textit{prop}(i) = \max \{1-(1+\beta)\eta_i,\,\, (1-\beta)(1-\eta_i)\}$, where $0\le\beta\le1$ is a parameter which can be set as 0.2.
However, in real-world applications especially for the \emph{structured label noise}, there may exist some classes with low noise rates but others with high ones, \eg, $\eta_i\ll \eta_j$ for some classes $i$, $j$.  Let $[p_1, \dots, p_c]$ denote the true class distribution, $[n_1, \dots, n_c]$ denotes the number of examples of each class in the noisy dataset.
Directly setting the selected number as $\textit{prop}(i)\cdot n_i$ may cause \emph{class distribution shift}, \ie, the distribution of $[\textit{prop}(1)\cdot n_1, \dots, \textit{prop}(c)\cdot n_c]$ seriously deviates from $[p_1, \dots, p_c]$.
To obey the true class distribution, let $m$ denote the total number of the potential selected data, we can set the selected number as $[p_1\cdot m,\dots, p_c\cdot m]$. Given the constraints $p_i \cdot m \le prop(i)\cdot n_i$, we can set $m=\min_{1\le i\le c}\{\frac{\textit{prop}(i)\cdot n_i}{p_i}\}$. However, since $p_i\cdot m$ may be much less than $\textit{prop}(i)\cdot n_i$ for some classes, many useful data may be wasted. Thus we additionally introduce a parameter $\gamma$ to achieve the trade-off between the class distribution unbiasedness and the full use of noisy data. The final selected number for the $i$-th class is set as
\begin{equation}\label{eq:number}
\textit{num}(i) = \min \{\gamma\cdot p_i\times m, \,\, \textit{prop}(i)\times n_i\}.
\end{equation}
Denote $\gamma_0 =1$ and $\gamma_1=\max_{1\le i\le c}\{\frac{\textit{prop}(i)\cdot n_i}{p_i\cdot m} \}$, if $\gamma = \gamma_0$, then $\textit{num}(i) = p_i\times m$; 
if $\gamma\ge\gamma_1$, then $\textit{num}(i)$ collapses to $\textit{prop}(i)\times n_i$. In practice, setting $\gamma = (\gamma_0+\gamma_1)/2$ may be a reasonable choice.

\begin{algorithm}[t]
	\renewcommand{\algorithmicrequire}{\textbf{Input:}}
	\renewcommand{\algorithmicensure}{\textbf{Output:}}
	\caption{RSL: Reformalization of Small-Loss criterion}
	\label{alg:selection}
	\begin{algorithmic}[1]
		\REQUIRE Noisy dataset $\tilde{D}$, 
		the initial model $g(\bm x;\Theta^{(0)})$, epoch limit $E$
		\FOR{$t=1$, $\dots$, $E$}
		     \STATE Update ${\Theta}^{(t-1)}$ on $\tilde{D}$ with one epoch to get ${\Theta}^{(t)}$; 
		     \STATE Calculate each example’s loss:
		       \STATE\quad $\forall (\bm x,\tilde{y})\in\tilde{D}$,\, $\ell_t(\bm x,\tilde{y})=\ell_{CE}( g(\bm x;{\Theta}^{(t)}),\tilde{y})$;
		\ENDFOR
		\STATE Calculate each example’s  mean loss:
		\STATE\quad $\forall (\bm x,\tilde{y})\in\tilde{D}$,\, $\bar{\ell}(\bm x,\tilde{y}) = \frac{1}{E}\sum_{t=1}^E \ell_t(\bm x,\tilde{y})$; 
		\FOR{$i=1$, $\dots$, $c$}
		\STATE $\tilde{D}_i = \{(\bm x, \tilde{y})\in \tilde{D}|\tilde{y}=i\}$;
		\STATE Rank examples in $\tilde{D}_i$  by $\bar{\ell}(\bm x,\tilde{y})$;
		\STATE Calculate $\textit{num}(i)$ according to Eq.~\eqref{eq:number};
		\STATE Select $\textit{num}(i)$ examples with smallest  $\bar{\ell}(\bm x,\tilde{y})$ as $S_i$; 
                \ENDFOR
                \STATE $D_\text{sel} = \cup_{i=1}^cS_i$;
		\STATE Train $g(\bm x; \Theta)$ with $D_\text{sel}$;
		\ENSURE The final classifier $g(\bm x; \Theta)$ 
	\end{algorithmic}  
\end{algorithm}

Based on the small-loss criterion, we can select a part of low-risk data from noisy data. Many previous methods only use the selected data and directly abandon the unselected data, but additional benefits could be obtained by treating them as unlabeled data and utilizing some semi-supervised learning methods. 
We utilize a general and representative semi-supervised learning method MixMatch~\cite{berthelot2019mixmatch} to verify the potentiality of this framework~\cite{wang2020seminll}.
Due to space limit, we put the technical details of standard MixMatch method $(\mathcal{L}', \mathcal{U}') = \text{MixMatch}(\mathcal{L}, \mathcal{U})$ in Appendix~C while emphasizing the adaptation we made for applying it in our case herein.
Note that the selected examples (denoted by $D_\text{sel}$ in Algorithm~\ref{alg:selection}) may still have a little label noise.
We propose to reweigh the selected examples to alleviate the influence of label noise.
Let $S_i$ denote the set of selected examples for the $i$-th class, $\bar{\ell}(\bm x,\tilde{y})$ denotes the mean loss of $(\bm x, \tilde{y})$, $\ell_*(i)=\min_{(\bm x,\tilde{y})\in S_i} \bar{\ell}(\bm x,\tilde{y})$ and $\ell^*(i)=\max_{(\bm x,\tilde{y})\in S_i} \bar{\ell}(\bm x,\tilde{y})$, for each $(\bm x, \tilde{y})\in S_i$, we set its weight as 
$ w(\bm x, \tilde{y}) = \exp(-\kappa\frac{\bar{\ell}(\bm x,\tilde{y}) -\ell_*(i)}{\ell^*(i) - \ell_*(i)}) $,
where $\kappa\ge 0$ is a hyperparameter, and bigger $\kappa$ implies that we assign smaller weights for examples with large losses. We embed the weight into MixMatch with weighted resampling technique and call it Weighted\_MixMatch. Denote $D_u = \{\bm x|\,\forall\, (\bm x,\tilde{y})\in \tilde{D}\backslash D_\text{sel}\}$, the above process can be formulated as $D_\text{sel\_WM} = \text{Weighted\_MixMatch}(D_\text{sel}, D_u)$. Then we name the method of training $g(\bm x;\Theta)$ with $D_\text{sel\_WM}$ by Weighted\_MixMatch (rather than $D_\text{sel}$) as RSL\_WM. 
We adopt the default hyper-parameters of standard MixMatch and additionally analyze the influence of $\kappa$ in experiments.

\begin{figure}[t]
\centering
 \includegraphics[width=\linewidth]{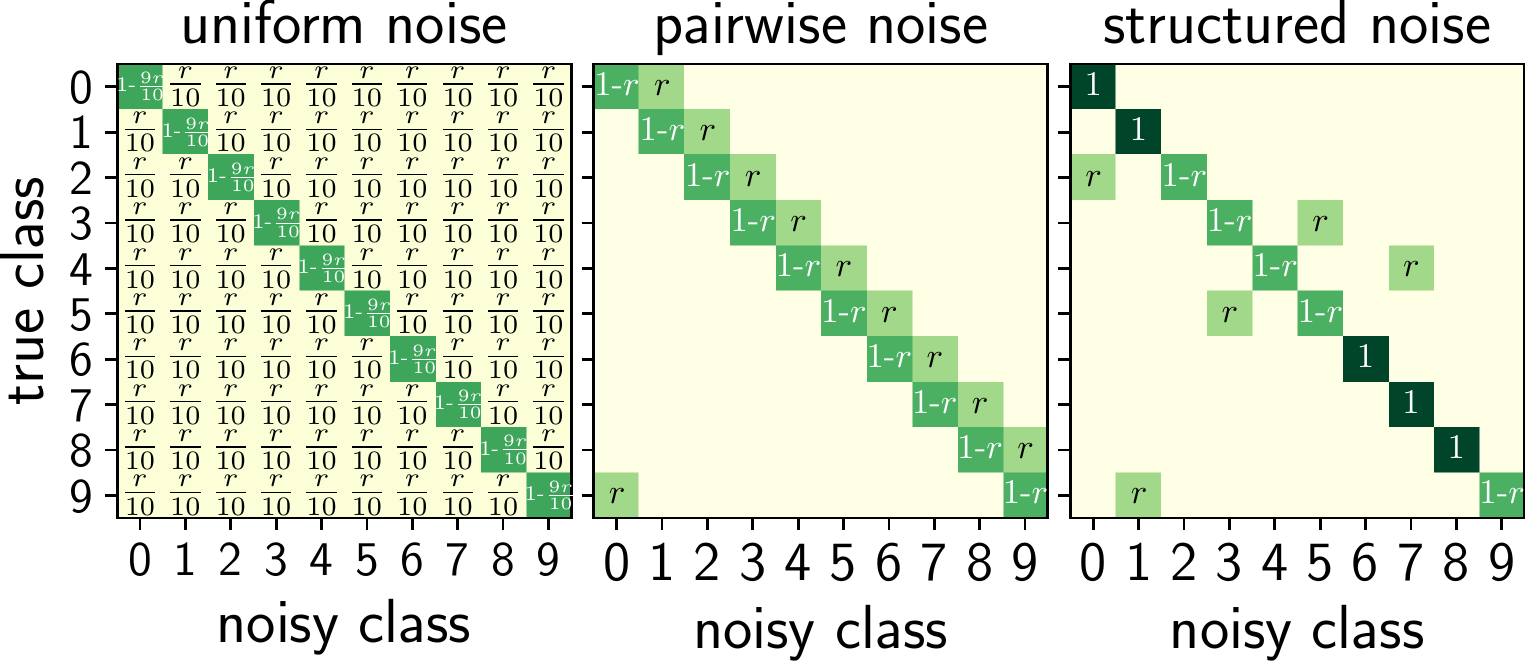}
   \caption{Three representative types of label noise on CIFAR-10.}
\label{fig:noise_matrix}
\end{figure}

\begin{figure*}[t]
\centering
\subfigure[uniform label noise]{
\centering
\includegraphics[width=0.32\linewidth]{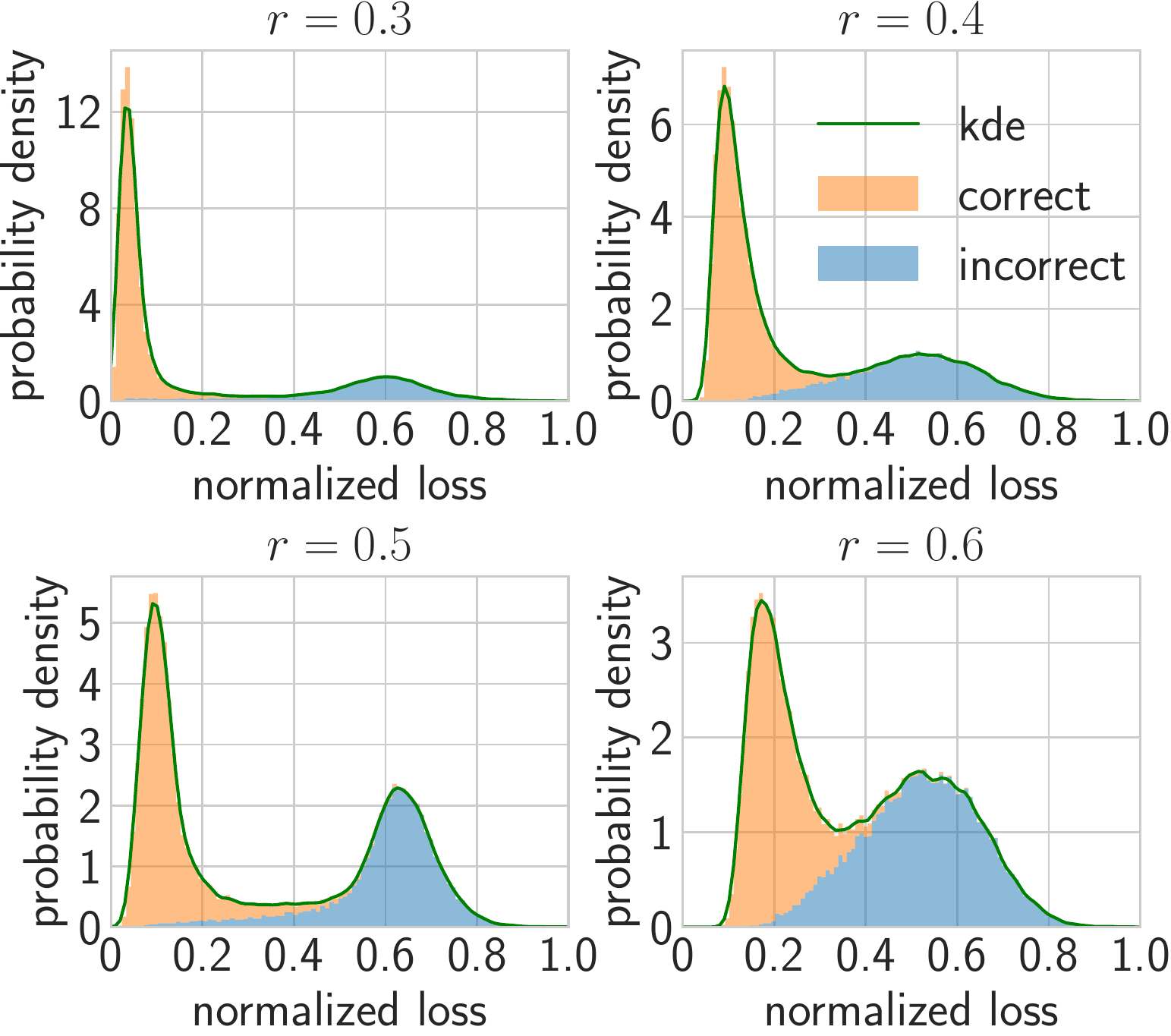}
\label{fig:loss_distribution_a}
}
\subfigure[pairwise label noise ]{
\centering
\includegraphics[width=0.32\linewidth]{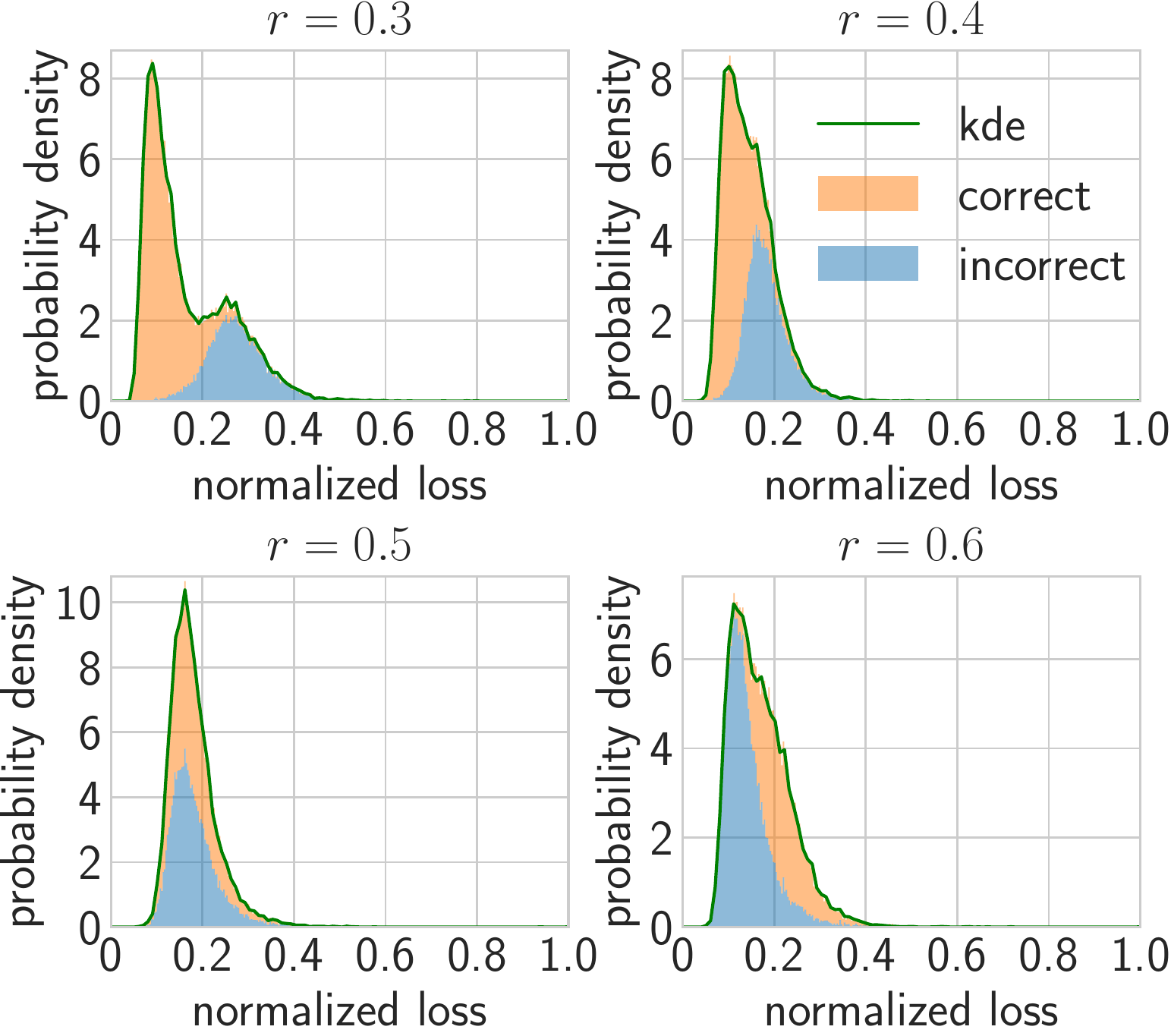}
\label{fig:loss_distribution_b}
}
\subfigure[structured label noise ]{
\includegraphics[width=0.32\linewidth]{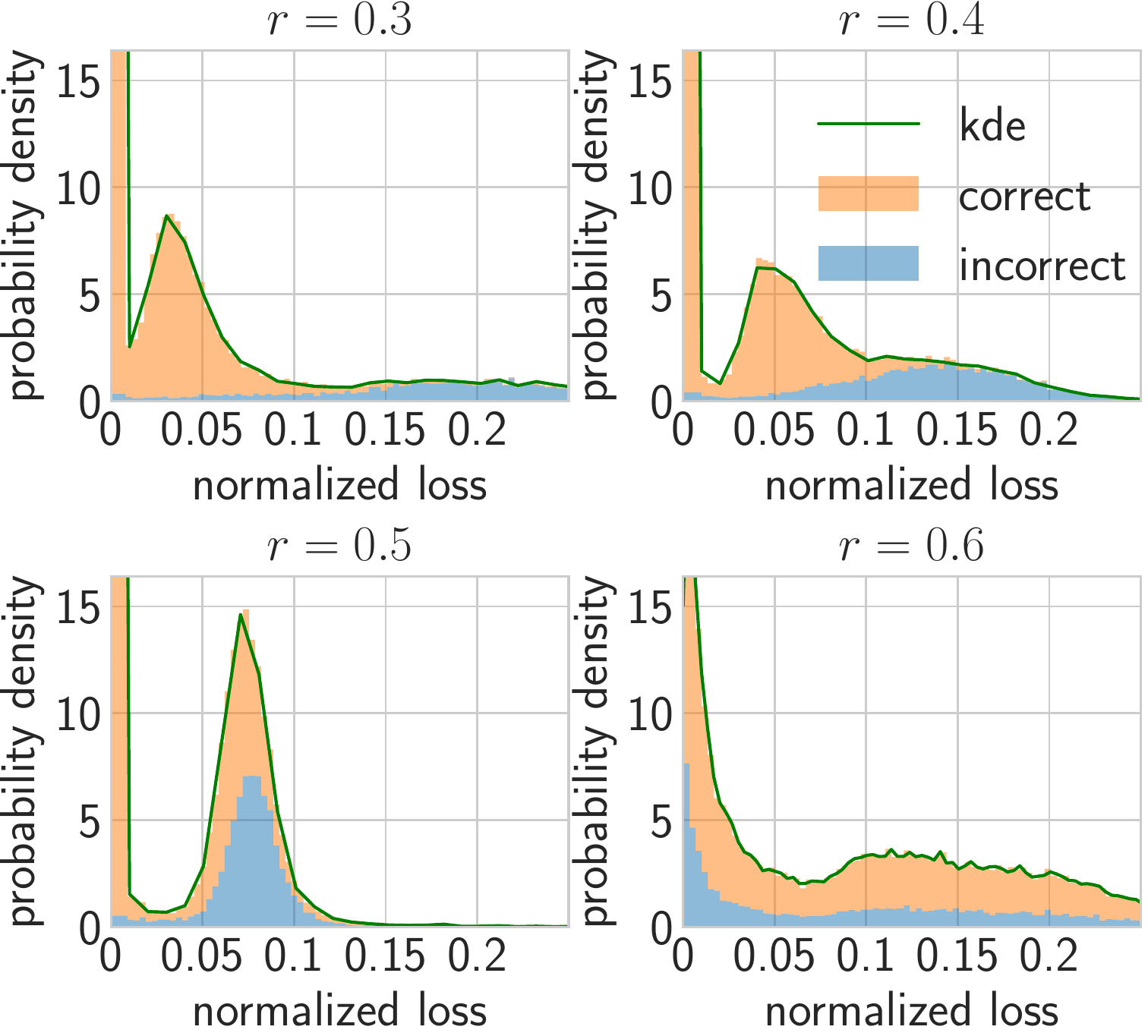}
\label{fig:loss_distribution_c}
}

\caption{Loss distribution of noisy examples on CIFAR-10 with different noise types and levels. The `kde’ represents kernel density estimation of the loss distribution. For (b) $r=0.5$, $r=0.6$, and (c) $r=0.5$, $r=0.6$, the diagonally-dominant condition is not satisfied.}
\label{fig:loss_distribution}
\end{figure*}

\section{Experiments}\label{section:experiment}
In this section, we conduct experiments on synthetic and real-world datasets to verify our theoretical explanation and the reformalization of the small-loss criterion RSL and RSL\_WM. 

\emph{Datasets.} The CIFAR-10/100 datasets contain 50K (10K) images for training (test). We retain 5K of the training set for validation following~\cite{tanaka2018joint}.
For label noise addition, we follow the common criteria in previous works~\cite{patrini2017making,han2018co}. Uniform noise is introduced by randomly flipping each label to one of 10/100  classes for CIFAR-10/100 with probability $r$. For pairwise noise, each class is circularly flipped to the next class with probability $r$ for CIFAR-10/100. We additionally synthesize structured noise for CIFAR-10 following~\cite{patrini2017making}, which mimics realistic noise taking place in similar classes: truck $\to$ automobile, bird $\to$ airplane, deer $\to$ horse, cat $\leftrightarrow$ dog with transition parameter $r$. The sketch map for these label noise types for CIFAR-10 is shown in Figure~\ref{fig:noise_matrix}.
The WebVision~\cite{Li2017a} dataset contains 2.4M noisy labeled images crawled from Flickr and Google by using 1,000 concepts in ILSVRC-2012~\cite{deng2009imagenet} as queries and the overall noise rate is rough $20\%$.
Following~\cite{chen2019understanding}, we use the first 50 classes of Google image subset for training and test on the corresponding $50$ classes of WebVision and ILSVRC-2012 validation set.

\begin{figure}[t]
\centering
\includegraphics[width=\linewidth]{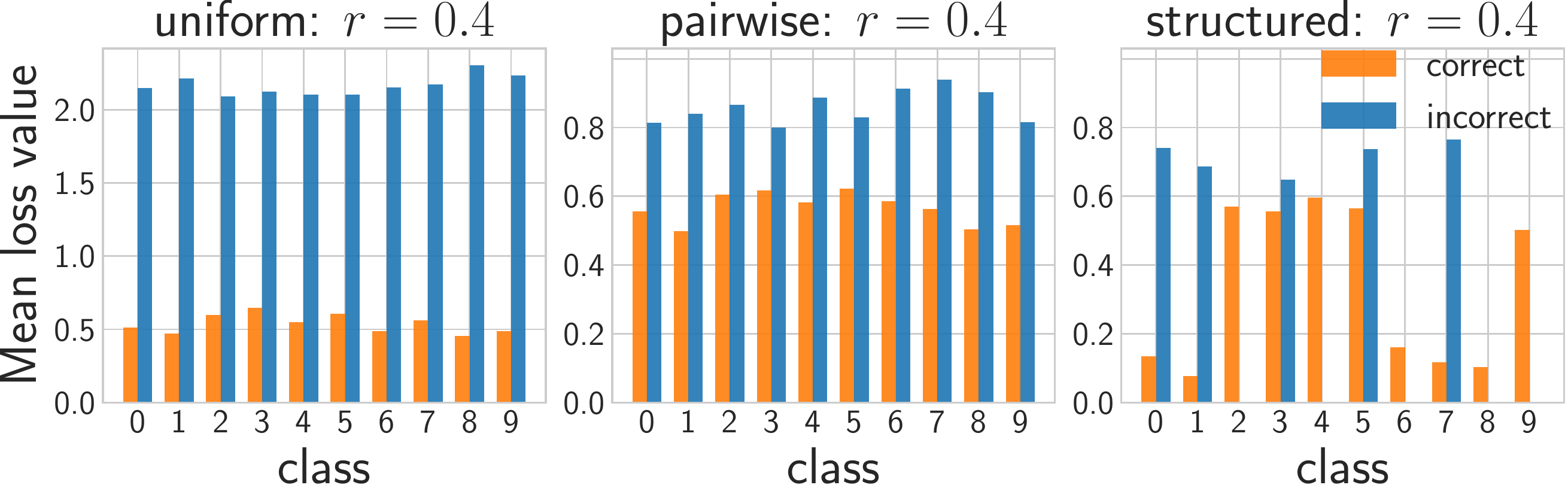}
\caption{Mean values of the mean loss of correct examples and incorrect ones for each class. For structured noise ($r=0.4$), some classes do not have label noise.}	
\label{fig:class-by-class}
\end{figure}

\begin{figure}[t]
\centering
\includegraphics[width=\linewidth]{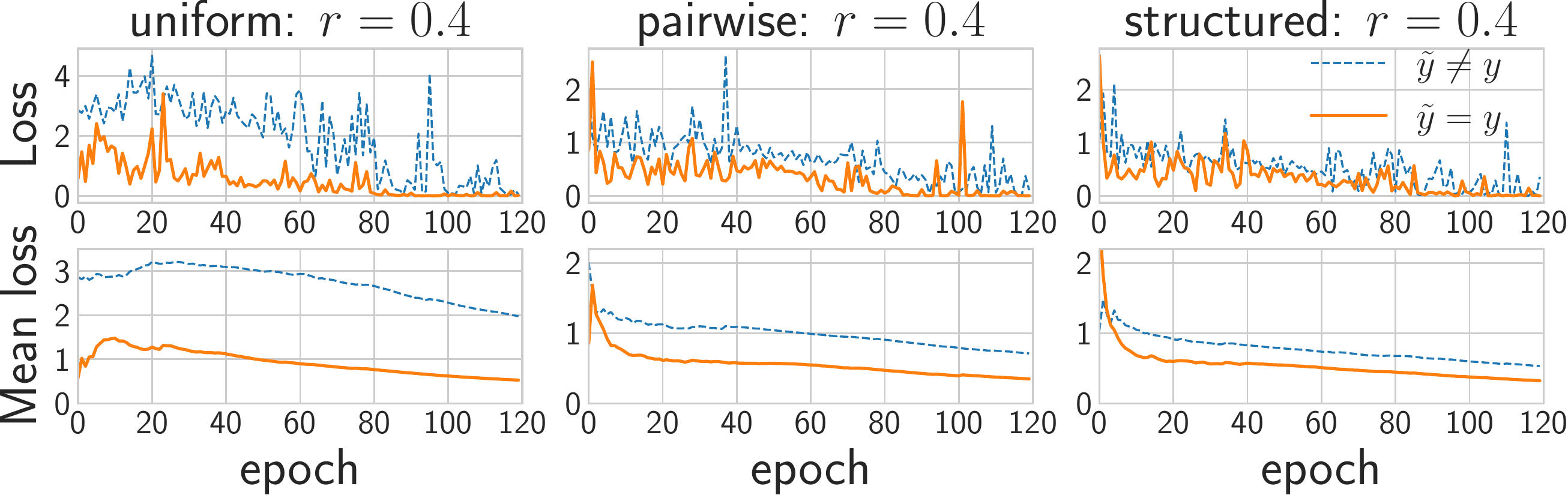}
\caption{Each epoch's loss and the cumulative mean loss for randomly chosen one pair of correct example and incorrect example.}
\label{fig:mean_vs_one_epoch}
\end{figure}

\begin{figure}[h] 
\centering
\includegraphics[width=\linewidth]{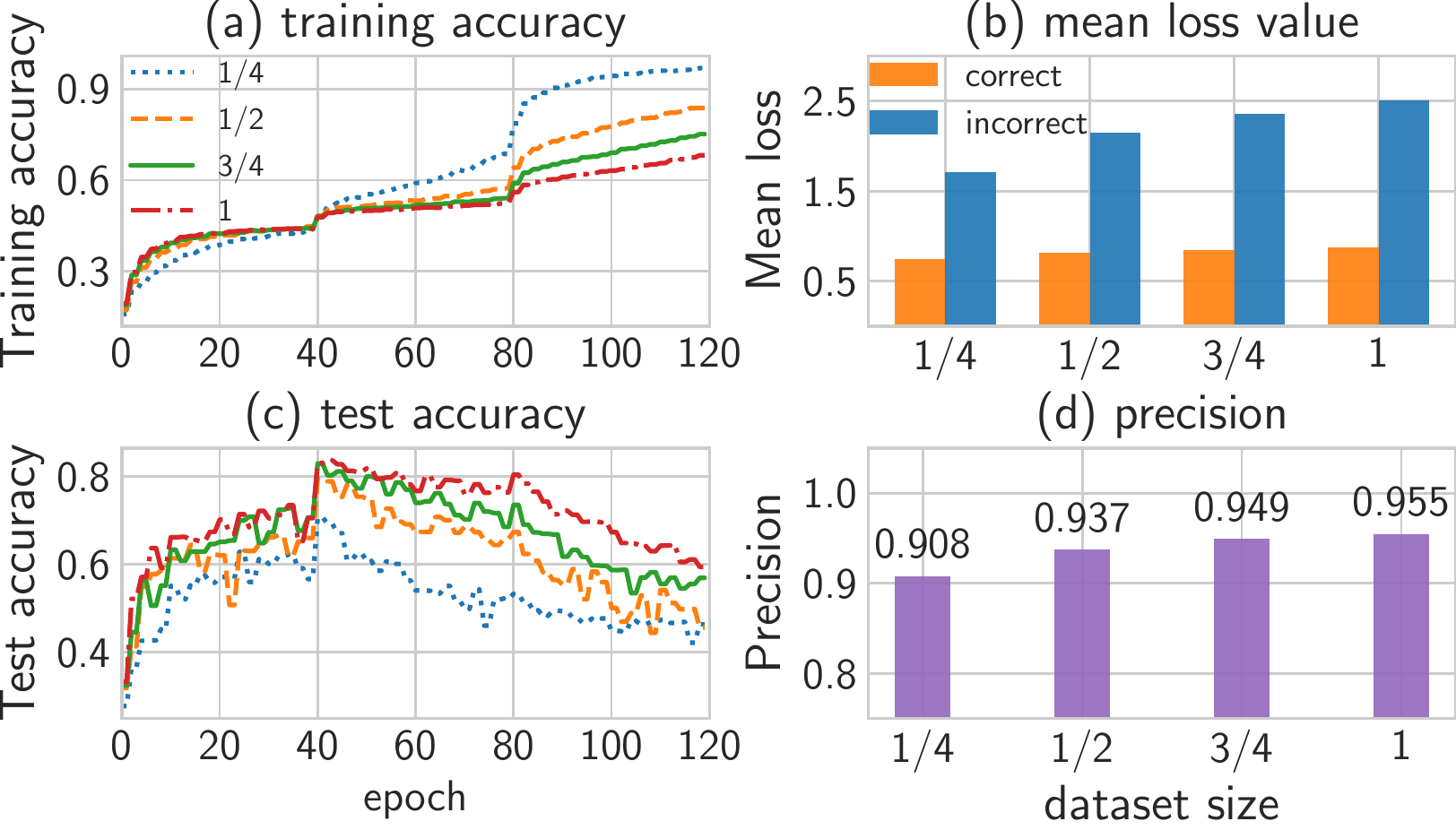}
\caption{Training accuracy and test accuracy of each epoch in the training process, the mean loss values for correct examples and incorrect ones, and the precision of the selected data (all with the same selection ratio) when using $\frac{1}{4}$, $\frac{1}{2}$, $\frac{3}{4}$ and $1$ of the original noisy training set (uniform  noise: $r=0.4$) respectively.}
\label{fig:different_data}
\end{figure}

\emph{Baselines.}
In this paper, we mainly contrast to state-of-the-art sample selection-based methods: Co-teaching~\cite{han2018co}, Co-teaching+~\cite{yu2019does}, INCV~\cite{chen2019understanding} and  JoCoR~\cite{wei2020combating}.
We also implement ``Cross Entropy'' method, which does not use any special treatments for label noise.
Additionally, we also compare with $\mathcal{L}_{\text{DMI}}$~\cite{xu2019l_dmi} on CIFAR-10. Since $\mathcal{L}_{\text{DMI}}$ can not be used for datasets with a large number of classes, we compare with Truncated $\mathcal{L}_q$~\cite{zhang2018generalized} on CIFAR-100.

\subsection{Empirical Findings} 
\label{subsection:verify-theory}

We verify our theoretical explanation with the following empirical findings on synthetic CIFAR-10 with different noise types and levels.
Due to space limit, detailed experimental settings and more  discussions can be found in Appendix~B.

Figure~\ref{fig:loss_distribution} implies that when the diagonally-dominant condition is not satisfied, incorrect examples may even have relative smaller loss than correct ones (see Figure~\ref{fig:loss_distribution} (b) $r=0.5$ and $r=0.6$, (c) $r=0.5$ and $r=0.6$), which verifies the necessity of this condition for learning from noisy labels with small-loss criterion.
Figure~\ref{fig:class-by-class} shows that the losses for different classes are not comparable, especially for more realistic structured label noise, which justifies the necessity of ranking the losses of examples class by class when selecting small-loss examples. 
Figure~\ref{fig:mean_vs_one_epoch} shows that there exist large fluctuations for single epoch's loss, while for the mean loss, the value of the correct example ($\tilde{y}=y$) is steadily smaller than that of the incorrect one ($\tilde{y}\neq y$). It justifies the effectiveness of using the mean loss for small-loss criterion. 
Theorem~\ref{thm:weak-small-loss} implies that the distance between the model $g$ and $g^*$ is important. In general, when trained with a bigger training set sampled from the noisy data distribution, the model $g$ will be closer to $g^*$. Thus we train models using different sizes of the training set to simulate different $g$ with different distances to $g^*$, and show the training dynamics and performances in Figure~\ref{fig:different_data}.
It can be found when $g$ is closer to $g^*$, the loss gap is bigger and the precision of the selected data is higher, which verifies that the small-loss criterion will have better utility when the model $g$ is closer to $g^*$. 
To further verify the effectiveness of the reformalization of small-loss criterion,  we compare the precision of the selected data with different methods.
In Algorithm~\ref{alg:selection}, we propose to select less data than $1-\eta$ for safety, where $\eta$ is the overall noise rate. 
But in here for fair comparison, we set the selected proportion as $1-\eta$ for all methods. 
Since the selected proportions of INCV and Co-teaching+ depend on the algorithms and can not be manually set, we only compare with Co-teaching and JoCoR. 
As shown in Table~\ref{tab:precision}, by comparing ``Only Mean Loss’’ and ``Our Method’’ with Co-teaching and JoCoR on all noise settings especially for structured noise (CIFAR-10), the effectiveness of ``mean loss’’ and ``selecting class by class’’ are verified respectively.

\begin{table*}[tbp]
\centering
\resizebox{\textwidth}{!}{%
\begin{tabular}{c|ccccc|cccc|cccc|cccc|cccc}
\hline\hline
\multirow{2}{*}{Method}  & \multicolumn{13}{c|}{CIFAR-10}                                                                                                                                                                                             & \multicolumn{8}{c}{CIFAR-100}                                                                                                        \\ \cline{2-22} 
                         & \multicolumn{5}{c|}{uniform noise}                                                 & \multicolumn{4}{c|}{pairwise noise}                               & \multicolumn{4}{c|}{structured noise}                             & \multicolumn{4}{c|}{uniform noise}                                & \multicolumn{4}{c}{pairwise noise}                               \\ \hline
noise parameter $r$ (\%) & 10             & 30             & 50             & 70             & 90             & 10             & 20             & 30             & 40             & 10             & 20             & 30             & 40             & 20             & 40             & 60             & 80             & 10             & 20             & 30             & 40             \\ \hline
Co-teaching              & 98.58          & 95.32          & 92.25          & 85.32          & 32.68          & 98.01          & 95.61          & 93.42          & 81.36          & 98.19          & 96.55          & 95.04          & 90.58          & 96.30          & 92.34          & 85.37          & 33.38          & 92.38          & 86.73          & 77.30          & 65.21          \\ \hline
JoCoR                    & 98.89          & 96.03          & 93.14          & 86.47          & 20.29          & 98.27          & 96.29          & 93.75          & 82.73          & 98.47          & 96.85          & 95.26          & 91.69          & 96.64          & 92.62          & 86.70          & 40.78          & 94.42          & 88.39          & 79.89          & 67.71          \\ \hline
Only Mean Loss           & 99.01          & 97.46          & 94.59          & 88.31          & 46.72          & 98.74          & 97.33          & 94.12          & 84.60          & 98.58          & 97.08          & 95.37          & 91.98          & 97.01          & 93.43          & 87.87          & 65.11          & 95.80          & 89.72          & 81.22          & 68.77          \\ \hline
Our Method               & \textbf{99.09} & \textbf{97.47} & \textbf{94.65} & \textbf{88.38} & \textbf{46.91} & \textbf{98.81} & \textbf{97.43} & \textbf{94.69} & \textbf{84.68} & \textbf{99.81} & \textbf{99.38} & \textbf{97.97} & \textbf{94.96} & \textbf{97.22} & \textbf{93.70} & \textbf{88.34} & \textbf{66.12} & \textbf{95.98} & \textbf{90.29} & \textbf{82.00} & \textbf{69.69} \\ \hline
\end{tabular}%
}
\caption{The precision (\%) of the selected data with different methods on CIFAR-10 and CIFAR-100. ``Only Mean Loss’’ represents using mean loss  but not selecting examples class by class. ``Our Method’’ represents using mean loss and selecting examples class by class.}
\label{tab:precision}
\end{table*}

\begin{table*}[ht]  
\centering
\resizebox{0.8\textwidth}{!}{%
\begin{tabular}{c|c|ccccc|cccc|cccc}
\hline
\hline
\multicolumn{2}{c|}{Method}  & \multicolumn{5}{c|}{uniform noise} & \multicolumn{4}{c|}{pairwise noise} & \multicolumn{4}{c}{structured noise} \\ \hline
\multicolumn{2}{c|}{noise parameter $r$ (\%)} & 10 & 30 & 50 & 70 & 90 & 10 & 20 & 30 & 40 & 10 & 20 & 30 & 40 \\ \hline
\multirow{2}{*}{Cross Entropy} & best & 91.24 & 88.30 &84.85  & 78.13 & 44.90 & 91.32 & 90.83 & 88.96 & 83.20& 91.80 & 90.95& 88.87 &86.57  \\
& last & 86.70 & 72.12 & 55.24 & 32.97 & 19.45 & 85.59 & 78.83 & 67.70 &  56.12 & 89.70 & 84.89 & 80.26 & 75.76  \\ \hline
\multirow{2}{*}{$\mathcal{L}_{\text{DMI}}$} & best &  90.47 & 87.76 & 84.12 & 77.85 & 36.71 & 91.13 & 90.90 & 89.12 & 85.56  & 91.14 &  90.19 & 88.41 &86.72  \\
& last & 90.07 & 87.74 & 84.10 & 77.73 & 36.37 & 91.03 & 90.45 & 88.87 & 85.32 & 90.28 & 89.43 & 88.13 & 86.25 \\ \hline
 \multirow{2}{*}{Co-teaching} & best & 90.60 & 89.83 & 85.14 & 65.76  & 11.70   & 91.59 & 89.42 &87.37  &78.18 & 90.72 & 89.67& 87.12 & 80.59 \\
& last & 90.36 & 88.98 & 85.09  & 65.65 & 11.69 & 90.71 & 89.02 & 87.24 & 71.76 & 90.35 & 89.63  & 86.73 & 77.81 \\ \hline
 \multirow{2}{*}{Co-teaching+}& best & 90.93 & 89.98 & 86.52 & 77.44 & 10.74 & 91.52 &90.22& 87.55  & 82.15  & 91.28 & 90.29& 88.17 &81.46  \\
& last & 90.90 & 89.36 & 86.48 & 77.38 & 10.54 & 91.30 & 89.37& 87.30 & 81.47 & 90.65 &  90.05 & 87.44   & 80.28  \\ \hline
\multirow{2}{*}{INCV} & best & 91.82& 90.72  & 86.34 & 73.11 & 38.38 & 91.42 & 89.26 & 87.84 &85.73 & 91.85 & 90.58 &  87.89 & 86.43 \\
& last & 91.79 & 89.48 & 86.43 & 72.78 & 38.29 & 91.37 & 89.19 & 87.50 & 85.18 & 91.62 & 90.14 & 87.68 & 86.23 \\ \hline
\multirow{2}{*}{JoCoR} & best & 92.30 & 89.52 &  87.27 & 79.57 & 26.38 & 91.87 & 90.38 & 88.42  & 83.48 & 92.02  & 90.87 & 88.78 & 83.59 \\& last & 92.28 & 89.48 & 85.86 & 79.62 &  25.18  & 91.82 & 90.32 & 87.44 & 83.42 & 91.99 & 90.23 & 88.04 & 83.40\\ \hline
\multirow{2}{*}{RSL} & best & 93.32& 91.34 & 88.21 & 82.21 & 39.75 & 92.71 & 91.13 & 90.51 & 86.73 & 92.58 & 91.32 & 89.97 &87.91  \\
& last & 93.23 & 91.13 & 87.93 & 82.08 & 39.54 & 92.47 & 90.89 & 90.31 & 86.57 & 92.42 & 91.24 & 89.83 & 87.85  \\ \hline
\multirow{2}{*}{RSL\_{WM}} & best & \textbf{94.15} & \textbf{93.78} & \textbf{93.38} & \textbf{91.51} & \textbf{48.33} & \textbf{94.08}  & \textbf{93.73} & \textbf{93.40} & \textbf{89.27} & \textbf{93.57} & \textbf{93.12} & \textbf{92.78} & \textbf{91.17} \\
& last & \textbf{93.59} & \textbf{93.42} & \textbf{93.27} & \textbf{91.31} & \textbf{47.43} & \textbf{93.21} & \textbf{93.19} & \textbf{93.10} & \textbf{88.85} & \textbf{93.33} & \textbf{92.83} & \textbf{92.34} & \textbf{90.63} \\ \hline
\end{tabular}%
}
\caption{The accuracy (\%) results on CIFAR-10. The term ``best’’ means the test accuracy of the epoch when validation accuracy is maximum, and ``last’’ means the test accuracy of the last epoch.}
\label{tab:CIFAR-10}
\end{table*}

\begin{table}[tbp]
\centering
      \resizebox{0.47\textwidth}{!}{%
	\begin{tabular}{c|c|cccc|cccc}
\hline\hline
\multicolumn{2}{c|}{Method}             & \multicolumn{4}{c|}{uniform noise} & \multicolumn{4}{c}{pairwise noise} \\ \hline
\multicolumn{2}{c|}{noise parameter $r$ (\%)} & 20      & 40      & 60      & 80     & 10      & 20      & 30      & 40      \\ \hline
\multirow{2}{*}{ Cross Entropy}         & best    & 62.61      & 53.00   & 42.74 & 29.08 &     68.18   & 64.31 & 59.05 &  45.70      \\
                                               & last    &  57.44   &  41.96  & 26.05 & 12.76  &  67.24   & 61.13 & 54.03 & 44.44        \\ \hline
\multirow{2}{*}{ Truncated $\mathcal{L}_q$}  & best    & 67.41 & 62.77 & 54.60  &19.47  &  68.93   & 67.36 & 62.21 & 46.89        \\
                                               & last    & 66.48  &  62.28  & 53.48 & 17.48  &  68.80 & 67.06 &62.12 & 45.97        \\ \hline
\multirow{2}{*}{Co-teaching}         & best    & 69.94 & 63.65   & 54.64  & 12.75  &   68.74    &     67.91  &    62.66      &   50.44   \\
                                                 & last    & 69.53       &   63.23   & 53.57 & 11.27  & 68.46      &   66.24  & 61.84          &   48.83     \\ \hline
\multirow{2}{*}{Co-teaching+}         & best    &  65.43     &  63.21  & 54.33 &  11.52 &  67.53      &   64.83     &      59.75     & 46.33       \\
                                                & last    & 64.74        &   62.69  & 52.23  & 10.57 &    67.37  &    64.26     &     58.59        &   45.67      \\ \hline
\multirow{2}{*}{INCV}         & best    &   62.68    &  59.78  & 41.39 & 23.43 &  63.93      & 56.68      &  50.87         & 38.95       \\
                                               & last    &  62.65   &  59.69  & 41.24 & 23.32  & 63.87     &  56.48     &     50.81        & 38.84        \\ \hline
\multirow{2}{*}{JoCoR}         & best    &   71.40    & 66.80   & 58.40 & 23.44  &    72.31    &  67.92     & 63.38          &    \textbf{54.37}    \\
                                               & last    &  70.62   &  66.10  & 57.65 & 23.36  & 71.81    &   67.32    &     62.79       &   \textbf{53.74}      \\ \hline
\multirow{2}{*}{RSL}        & best    & 72.12      &  67.23  &  59.24 & 38.32  &      72.42 &  68.43     &     62.45      &  53.62      \\
                                               & last    &  71.84   &  67.03  & 58.78 & 38.04  &  72.46   &    68.27   &        62.23     & 53.25        \\ \hline
\multirow{2}{*}{RSL\_WM}        & best    &   \textbf{74.88}     &  \textbf{71.51}    &  \textbf{67.25}  & \textbf{49.58} &  \textbf{74.48}      &     \textbf{71.18}    & \textbf{64.67}   & 54.34     \\
                                            & last    &  \textbf{73.92}     &  \textbf{70.69}    & \textbf{66.07} &  \textbf{49.17} &   \textbf{73.77}    &   \textbf{70.54}    &   \textbf{63.87}    &  53.65     \\ \hline
\end{tabular}%
}
\caption{The accuracy (\%) results on CIFAR-100.} 
     \label{tab:CIFAR-100}
\end{table}

\subsection{Evaluation on Benchmark Datasets}\label{subsection:method}
We evaluate the performance of the proposed method on benchmark datasets.
As that in Co-teaching, Co-teaching+ and JoCoR, we also assume that the noise rates are known and put sensitivity analysis of our method on noise rates in Appendix~C.
We use the default parameters $\beta=0.2$, $\gamma =(\gamma_0 + \gamma_1)/2$
 and $\kappa=-\log(0.7)$ in experiments.
Due to space limit, we put all the implementation details and sensitivity analysis on $\beta$, $\gamma$ and $\kappa$ in Appendix~C. 
For CIFAR-10 and CIFAR-100, as shown in Tables~\ref{tab:CIFAR-10} and~\ref{tab:CIFAR-100}, RSL\_WM achieves the best performance in almost all noise settings. Although we select less data than Co-teaching and JoCoR, learning only with the selected data (RSL) still gets better accuracy than the compared methods in most cases. 
By exploiting the information of the unselected data with Weighted\_MixMatch, better performance is achieved, which verifies the effectiveness of this framework. In the pairwise noise ($r=0.4$) on CIFAR-100, our method is less competitive due to we select much less data than JoCoR.
For WebVision, we select the first $76\%$ ($\beta=0.2$) examples with small mean loss for each class. 
We also compare with F-correction~\cite{patrini2017making}, MentorNet~\cite{jiang2018mentornet}, and D2L~\cite{ma2018dimensionality} since they have the same experimental setting as ours. 
As shown in Table~\ref{tab:WebVision}, our method achieves better performance than the compared methods.

\begin{table}[t]
\centering
\resizebox{0.35\textwidth}{!}{%
\begin{tabular}{c|cc|c|c}
\hline
\hline
\multirow{2}{*}{Method} & \multicolumn{2}{c|}{WebVision Val.} & \multicolumn{2}{c}{ILSVRC2012 Val.} \\ \cline{2-5}   
 & \multicolumn{1}{c|}{top1} & top5 &  \multicolumn{1}{c|}{top1} & top5 \\ \hline
Cross Entropy & \multicolumn{1}{c|}{58.24} & 79.26 & 54.83 & 77.70 \\ \hline
F-correction & \multicolumn{1}{c|}{61.12} & 82.68 & 57.36 & 82.36 \\ \hline
Co-teaching & \multicolumn{1}{c|}{63.58} & 85.20 & 61.48 & 84.70 \\ \hline
MentorNet & \multicolumn{1}{c|}{63.00} & 81.40 & 57.80 & 79.92 \\ \hline
D2L & \multicolumn{1}{c|}{62.68} & 84.00 & 57.80 & 81.36 \\ \hline
Co-teaching+& \multicolumn{1}{c|}{63.21} & 84.78 & 61.32 & 83.52 \\ \hline
INCV & \multicolumn{1}{c|}{65.24} & 85.34 & 61.60 & 84.38 \\ \hline
JoCoR & \multicolumn{1}{c|}{65.28} & 85.38 & 61.54&84.46 \\ \hline
RSL & \multicolumn{1}{c|}{65.64} & 85.72 & 62.04 & 84.84 \\ \hline
RSL\_WM & \multicolumn{1}{c|}{\textbf{66.56}} & \textbf{86.54} & \textbf{63.40} & \textbf{85.43} \\ \hline
\end{tabular}%
}
\caption{The accuracy (\%) results on WebVision.}
\label{tab:WebVision}
\end{table}

\section{Conclusion}

In this paper, we establish the connection between noisy data distribution and the small-loss criterion and theoretically explain why the widely-used small-loss criterion works.
In the future, we will consider extending the theoretical explanation to the instance-dependent label noise~\cite{xia2020part}.

\section*{Acknowledgments}
This work is supported by the National Key Research and Development Program of China (2017YFB1002201), the National Science Foundation of China (61921006, 61673202), and the Collaborative Innovation Center of Novel Software Technology and Industrialization.

\setcounter{figure}{5}
\setcounter{table}{4}
\setcounter{theorem}{0}

\begin{appendices}
\setcounter{figure}{5}
\setcounter{table}{4}
\setcounter{theorem}{0}
\setcounter{lemma}{0}
\section{Proof}\label{appendix:proof}

\begin{lemma}\label{lemma:1988} 
If $T$ satisfies the row-diagonally dominant condition $T_{ii}>\max_{j\neq i}T_{ij}$, $\forall i$, then the target concept $f^*$ has the minimum expected 0-1 loss on the noisy data, i.e., $\forall \, f\neq f^*$, $\mathbb E_{(\bm x,\tilde{y})}[\ell_{01}(f^*(\bm x), \tilde{y})] \le \mathbb E_{(\bm x,\tilde{y})}[\ell_{01}(f(\bm x), \tilde{y})]$.
\end{lemma}
\begin{proof} 
$\forall\, f:\mathcal{X}\to\mathcal{Y}$, 
\[
\begin{split}
P(f(\bm x)\neq\tilde{y}) &= \mathbb{E}_{(\bm x, \tilde{y})} [\ell_{01}(f(\bm x), \tilde{y})]\\
{}&=\mathbb{E}_{(\bm x, \tilde{y})} \mathbb{I}[f(\bm x) \neq \tilde{y}]\\
{}&=\int_{\bm x\in \mathcal{X}}  \sum_{j=1}^c \mathbb{I}[f(\bm x) \neq j] p(\bm x, \tilde{y}=j) \mathrm{d}\bm x\\
{}&=\int_{\bm x\in \mathcal{X}}  \big(1-p(\tilde{y} = f(\bm x)|\bm x)\big)p(\bm x) \mathrm{d}\bm x.\\
\end{split}
\]

With the class-conditional noise assumption, we have $p(\tilde{y} = f(\bm x)|\bm x) =\sum_{i=1}^c p(\tilde{y} = f(\bm x), y=i|\bm x)=\sum_{i=1}^c p(\tilde{y} = f(\bm x)|y=i, \bm x)p(y=i|\bm x)=\sum_{i=1}^c p(\tilde{y} = f(\bm x)|y=i)p(y=i|\bm x)=\sum_{i=1}^c T_{if(\bm x)} p(y=i|\bm x)$, and thus
\[
\begin{split}
P(f(\bm x)\neq\tilde{y})&=\int_{\bm x\in \mathcal{X}}  \big(1-\sum_{i=1}^c T_{if(\bm x)} p(y=i|\bm x)\big)p(\bm x) \mathrm{d}\bm x\\
{}&=\int_{\bm x\in \mathcal{X}}  [1-T_{f^*(\bm x)f(\bm x)}] p(\bm x) \mathrm{d}\bm x,\\
\end{split}
\]
where the last equation is due to that each $\bm x$ has a unique true label $f^*(\bm x)$. When $f=f^*$ in the above equation, we have $P(f^*(\bm x)\neq\tilde{y}) =\int_{\bm x\in \mathcal{X}}  [1-T_{f^*(\bm x)f^*(\bm x)}] p(\bm x) \mathrm{d}\bm x$. When the matrix $T$ satisfies the row-diagonally dominant condition $T_{ii}>\max_{j\neq i}T_{ij}$, we have $1-T_{f^*(\bm x)f^*(\bm x)} < 1- T_{f^*(\bm x)j}$, $\forall\, j\neq f^*(\bm x)$, and thus for any $f\neq f^*$, we have $P(f^*(\bm x)\neq\tilde{y})\le P(f(\bm x)\neq\tilde{y})$, \ie, $\mathbb E_{(\bm x,\tilde{y})}[\ell_{01}(f^*(\bm x), \tilde{y})] \le \mathbb E_{(\bm x,\tilde{y})}[\ell_{01}(f(\bm x), \tilde{y})]$. More rigorously, if the measure of the area where $f$ not equals to $f^*$ is non-zero, \ie, $P(f(\bm x)\neq f^*(\bm x))>0$, we have $\mathbb E_{(\bm x,\tilde{y})}[\ell_{01}(f^*(\bm x), \tilde{y})] < \mathbb E_{(\bm x,\tilde{y})}[\ell_{01}(f(\bm x), \tilde{y})]$.
\end{proof}

\begin{lemma}\label{lemma:row-dominant}
Let $g^*$ denote the deep neural network minimizing the cross-entropy loss in Eq.~(1), the induced classifier $f_{g^*}$ satisfies $f_{g^*}(\bm x)=y$, $\forall \bm x\in\mathcal{X}$, if and only if $T$ satisfies the row-diagonally dominant condition $T_{ii}> \max_{j\neq i}T_{ij}$, $\forall i$.
\end{lemma}

\begin{proof}

For notational clarity, we denote the one-hot encoding of $\tilde{y}$ is $\tilde{\bm d}=[\tilde{d}_1, \dots, \tilde{d}_c]$, \ie, $\tilde{d}_{\tilde y} = 1$ and $\tilde{d}_i = 0$, $\forall\, i\ne \tilde y$.
For cross-entropy loss function, $\ell_{CE}(g(\bm x;\Theta), \tilde{y})= -\sum_{i=1}^c \tilde{d}_i\log(\hat{p}_i(\bm x))$.
Considering the expected cross-entropy loss on noisy data, we have
\[
\begin{split}
 \mathbb E&_{(\bm x,\tilde{y})}[\ell_{CE}(g(\bm x;\Theta), \tilde{y})] =  -\mathbb E_{(\bm x,\tilde{y})}\Big[\sum_{i=1}^c \tilde{d}_i\log(\hat{p}_i(\bm x))\Big]\\
 {} &=  -\int_{ \bm x\in\mathcal{X}} \sum_{j=1}^c  \Big[\sum_{i=1}^c \tilde{d}_i\log(\hat{p}_i(\bm x))\Big]p(\bm x, \tilde{y} = j) \mathrm{d}\bm x\\
    {} &=  -\int_{ \bm x\in\mathcal{X}} \Big[\sum_{i=1}^c \Big[\sum_{j=1}^c \tilde{d}_ip(\tilde{y}=j|\bm x) \Big]\log(\hat{p}_i(\bm x))\Big]p(\bm x) \mathrm{d}\bm x\\
{} &=  -\int_{ \bm x\in\mathcal{X}} \Big[\sum_{i=1}^c \mathbb E[\tilde{d}_i|\bm x]\log(\hat{p}_i(\bm x))\Big]p(\bm x) \mathrm{d}\bm x,\\
\end{split}
\]
where the last equation is due to $\mathbb E[\tilde{d}_i|\bm x] = \sum_{j=1}^c \tilde{d}_ip(\tilde{y}=j|\bm x)$. Note that when $\mathbb E_{(\bm x,\tilde{y})}[\ell_{CE}(g(\bm x; \Theta), \tilde{y})]$ is minimized, $-\sum_{i=1}^c\mathbb E[\tilde{d}_i|\bm x]\log(\hat{p}_i(\bm x))$ is also minimized for each $\bm x \in \mathcal{X}$. For cross-entropy loss, $0\le \hat{p}_i(\bm x)\le 1$ and $\sum_{i=1}^c \hat{p}_i(\bm x) = 1$. It is formalized as the following optimization problem:
\[
\begin{split}
\min_{[\hat{p}_1(\bm x), \dots, \hat{p}_c(\bm x)]} &\, -\sum_{i=1}^c\mathbb E[\tilde{d}_i|\bm x]\log(\hat{p}_i(\bm x)),\\
\mathrm{s.t.}\quad & \sum_{i=1}^c \hat{p}_i(\bm x) =1 ,\, 0\le \hat{p}_i(\bm x) \le 1.\\
\end{split}
\]
It can be found that $-\sum_{i=1}^c\mathbb E[\tilde{d}_i|\bm x]\log(\hat{p}_i(\bm x))$ is minimized when $ \hat{p}_i(\bm x) = \mathbb E[\tilde{d}_i|\bm x]$, $\forall\, 1\le i \le c$ by Lagrange multiplier method~\cite{boyd2014convex}. Furthermore, $\mathbb E[\tilde{d}_i|\bm x] = \sum_{j=1}^c \mathbb I[i=j]p(\tilde{y}=j|\bm x) = p(\tilde{y}=i|\bm x)$, and thus $\hat{p}_i(\bm x) = p(\tilde{y}=i|\bm x)$.
Thus for example $(\bm x, \tilde{y})$, we have $\hat{p}_k(\bm x) = \mathbb E[\tilde{d}_k|\bm x] = p(\tilde{y}=k|\bm x)$.

When the class-conditional noise assumption holds, we have
\[
\begin{split}
p(\tilde{y}=k|\bm x) &= \sum_{i=1}^c p(\tilde{y}=k, y=i|\bm x)\\
{}&=  \sum_{i=1}^c p(y=i|\bm x) p(\tilde{y}=k|y=i, \bm x) \\
{}&=  \sum_{i=1}^c p(y=i|\bm x) p(\tilde{y}=k|y=i) \\
{}&= p(\tilde{y}=k|y=f^*(\bm x)) \\
{}&= T_{f^*(\bm x)k},\\
\end{split}
\] 
where the third equation is due to the class-conditional noise assumption and the forth equation is due to that each $\bm x$ has a true label $f^*(\bm x)$. Thus we know that the softmax output of $g^*$ satisfies $\hat{p}_k(\bm x) = T_{f^*(\bm x)k}$ for $\bm x\in\mathcal{X}$. The prediction induced by the network $g^*$ is
\[f_{g^*}(\bm x)=\mathop{\arg\max}_k \hat{p}_k(\bm x).\]
It is easy to find that $f_{g^*}(\bm x)=y=f^*(\bm x)$ is equivalent to $T_{f^*(\bm x)f^*(\bm x)} >T_{f^*(\bm x)k}$, $\forall\, k\neq f^*(\bm x)$, \ie, $ T_{ii} > T_{ij}$, $\forall\, j\neq i$, by considering all classes. Thus Lemma~\ref{lemma:row-dominant} is proved.
\end{proof}

\begin{theorem}\label{thm:the-small-loss}
Let $g^*$ denote the deep neural network minimizing the cross-entropy loss in Eq.~(1), $(\bm x_1,\tilde{y})$ and $(\bm x_2,\tilde{y})$ are any two examples with the same observed label $\tilde{y}$ in $\tilde{D}$ satisfying that $f^*(\bm x_1)=\tilde{y}$ and $f^*(\bm x_2)\neq \tilde{y}$, 
if $T$ satisfies the diagonally-dominant condition $T_{ii} > \max\,\{\max_{j \neq i}T_{ij}, \,\,\max_{j\neq i}T_{ji}\}$, $\forall i$, then $\ell_{CE}(g^*(\bm x_1), \tilde{y})<\ell_{CE}(g^*(\bm x_2), \tilde{y})$. 
\end{theorem}

\begin{proof}
When the diagonally-dominant condition is satisfied, \ie, $\forall i$, $T_{ii} > \max\,\{\max_{j \neq i}T_{ij}, \,\,\max_{j\neq i}T_{ji}\}$, we first have $\forall i$, $T_{ii} > \max_{j \neq i}T_{ij}$. Thus according to Lemma~\ref{lemma:row-dominant}, the induced classifier $f_{g^*}$ satisfies $\forall \bm x\in\mathcal{X}$, $f_{g^*}(\bm x)=y$. Furthermore, from the proof process of Lemma~\ref{lemma:row-dominant}, we know that the softmax output of $g^*$ satisfies $\hat{p}_k(\bm x) = T_{f^*(\bm x)k}$ for $\bm x\in\mathcal{X}$.
Given any two examples $(\bm x_1,\tilde{y})$ and $(\bm x_2,\tilde{y})$ in $\tilde{D}$ with the same observed label $\tilde{y}$ satisfying that $f^*(\bm x_1)=\tilde{y}$ and $f^*(\bm x_2)\neq \tilde{y}$,  
the softmax ouput of $g^*$ for $(\bm x_1, \tilde{y})$ is \[
	[\hat{p}_1(\bm x_1), \dots, \hat{p}_{c}(\bm x_1)] =  [T_{f^*(\bm x_1)1}, \dots, T_{f^*(\bm x_1)c}].\]
	The loss value of $g^*$ on $(\bm x_1,\tilde{y})$ is
	\[\ell_{CE}(g^*(\bm x_1),\tilde{y})=-\log(\hat{p}_{\tilde{y}}(\bm x_1))=-\log(T_{f^*(\bm x_1)\tilde{y}}),\]
	and the loss value of $g^*$ on $(\bm x_2,\tilde{y})$ is 
	\[\ell_{CE}(g^*(\bm x_2),\tilde{y})=-\log(\hat{p}_{\tilde{y}}(\bm x_2))=-\log(T_{f^*(\bm x_2)\tilde{y}}).\]
	According to the diagonally-dominant condition, we also have $T_{\tilde{y}\tilde{y}}>\max_{j\neq \tilde{y}}T_{j\tilde{y}}$. Since $f^*(\bm x_1)=\tilde{y}$ while $f^*(\bm x_2)\neq\tilde{y}$, we have $\ell_{CE}(g^*(\bm x_1),\tilde{y})=-\log(T_{f^*(\bm x_1)\tilde{y}}) = -\log(T_{\tilde{y}\tilde{y}}) < -\log(T_{f^*(\bm x_2)\tilde{y}}) = \ell_{CE}(g^*(\bm x_2),\tilde{y})$.	
\end{proof}

\newpage
\begin{theorem}\label{thm:weak-small-loss}
Suppose $g$ is $\epsilon$-close to $g^*$, i.e., $\|g-g^*\|_\infty = \epsilon$, for two examples $(\bm x_1, \tilde{y})$ and $(\bm x_2, \tilde{y})$, assume $f^*(\bm x_1)=\tilde{y}$ and $f^*(\bm x_2)\neq \tilde{y}$, if $T$ satisfies the diagonally-dominant condition $T_{ii} > \max\,\{\max_{j \neq i}T_{ij}, \,\,\max_{j\neq i}T_{ji}\}$, $\forall i$, and $\epsilon < \frac{1}{2}\cdot(T_{\tilde{y}\tilde{y}}-T_{f^*(\bm x_2)\tilde{y}})$, then $\ell_{CE}(g(\bm x_1), \tilde{y}) < \ell_{CE}(g(\bm x_2), \tilde{y})$. 
\end{theorem}

\begin{proof}
With $\|g-g^*\|_\infty = \epsilon$, we have: $[g^*(\bm x)]_j - \epsilon \le [g(\bm x)]_j \le [g^*(\bm x)]_j + \epsilon,\, \forall\, \bm x \in \mathcal{X}, \,\forall\, 1\le j \le c$.
Thus $\ell_{CE}(g(\bm x_1), \tilde{y}) = -\log([g(\bm x_1)]_{\tilde{y}}) \le -\log([g^*(\bm x_1)]_{\tilde{y}} - \epsilon) = -\log(T_{\tilde{y}\tilde{y}} - \epsilon)$, and $\ell_{CE}(g(\bm x_2), \tilde{y}) = -\log([g(\bm x_2)]_{\tilde{y}}) \ge -\log([g^*(\bm x_2)]_{\tilde{y}} + \epsilon) = -\log(T_{f^*(\bm x_2)\tilde{y}} + \epsilon)$.
When $\epsilon < \frac{1}{2}\cdot(T_{\tilde{y}\tilde{y}}-T_{f^*(\bm x_2)\tilde{y}})$, we have $-\log(T_{\tilde{y}\tilde{y}} - \epsilon) \le -\log(T_{f^*(\bm x_2)\tilde{y}} + \epsilon)$, and thus $\ell_{CE}(g(\bm x_1), \tilde{y}) < \ell_{CE}(g(\bm x_2), \tilde{y})$.
Furthermore, actually we have
\[
\begin{split}
&\ell_{CE}(g(\bm x_2), \tilde{y})-\ell_{CE}(g(\bm x_1), \tilde{y})
\ge \log\big(\frac{T_{\tilde{y}\tilde{y}} - \epsilon}{T_{f^*(\bm x_2)\tilde{y}} + \epsilon}\big), \\
\end{split}
\]
\ie, the loss gap between $(\bm x_1, \tilde{y})$ and $(\bm x_2, \tilde{y})$ depends on $\epsilon$.
Considering all classes, when $\epsilon < \frac{1}{2}\cdot\min_{1\le i\le c} \big(T_{ii} - \max_{j\neq i}T_{ji}\big)$, for any examples with the same observed labels, the correct examples have smaller loss than the incorrect ones.
\end{proof}

\newpage
\section{Details of Empirical Findings}

For verification of theoretical explanation, we do empirical studies by training PreAct ResNet-32~\cite{tanaka2018joint} on synthetic CIFAR-10 with different noise types and levels. The image data is augmented by horizontal random flip and $32\times32$ random crops after padding with 4 pixels. For the optimizer, we deploy SGD with momentum of $0.9$ and a weight decay of $10^{-4}$. The batch size is set to $128$. The network is trained on the whole noisy dataset with an initial learning rate $0.2$ for CIFAR-10 and divided by 10 after 40 and 80 epochs (120 in total).

\subsection{Loss distribution}\label{appendix:loss-distribution}

Theorems~1 and 2 
implies that when the noise transition matrix $T$ satisfies the diagonally-dominant condition, for the examples with the same observed label, the correct examples will have smaller loss than that of incorrect ones. 
We consider uniform label noise, pairwise label noise and structured label noise with noise parameter $r=0.3$, $r=0.4$, $r=0.5$, $r=0.6$ on CIFAR-10, some of which satisfy diagonally-dominant condition while others do not satisfy this condition (cf. Figure~1).
We calculate each example's mean loss value $\bar{\ell}(\bm x, \tilde{y})$ and normalize them to $[0,1]$ by dividing the maximum of loss values for better representation. Besides, to exhibit the shape of loss distribution, we use kernel density estimation~\cite{bishop2006pattern} with gaussian kernel to fit  the distribution of loss values: 
\[p_{kde}(s) = \frac{1}{N}\sum_{i=1}^N\frac{1}{(2\pi h^2)^{1/2}}\exp\big\{-\frac{(s-s_i)^2}{2h^2}\big\},\]
where $h$ represents the standard deviation of the Gaussian components, $s_i$ represents the mean loss value of the $i$-th example along all training epochs and $N$ represents the total number of noisy examples.

As shown in Figure~2 in Section~5.1 of the main paper,  loss distributions of three noise types are quite different. For uniform label noise, the loss distribution presents a clear bimodal structure. For pairwise label noise, the loss distribution presents an unclear bimodal structure. And for structured label noise, the loss distribution shows a somewhat irregular shape, where many examples have very small loss values and the rest part has relatively large loss value, which is due to there exist some classes whose noise rates are zeros.

For uniform label noise, the diagonally-dominant condition holds for any noise parameter $r<1.0$. In Figure~2(a), for $r=0.3$, $r=0.4$, $r=0.5$ and $r=0.6$, the examples with correct label are relatively well separated from that with incorrect label. For pairwise noise (Figure~2(b)), the diagonally-dominant condition holds when noise parameter $r<0.5$. We can see that for $r=0.3$ and $r=0.4$, in general the examples with correct label have smaller loss than that with incorrect label. But for $r=0.5$ and $r=0.6$, many examples with incorrect label have smaller loss than that with correct label. For structured noise (Figure~2(c)), when $r\ge 0.5$, diagonally-dominant condition does not hold for some classes and some examples with incorrect label have smaller loss than that with correct label. The experimental results verify the necessity of diagonally-dominant condition in Theorems~1 and 2.

\subsection{The loss for different classes}
For uniform ($r=0.4$), pairwise ($r=0.4$) and structured ($r=0.4$) label noise, we plot the mean values of the mean loss of correct examples and incorrect ones for each class in Figure~3 in Section~5.1 of the main paper. 
It can be found that the losses for different classes are not comparable, especially for more realistic structured label noise, which justifies the necessity of ranking the loss of the examples \emph{class by class}. 

\subsection{Mean loss and single epoch’s loss}

To observe the characteristics of the single epoch’s loss and the mean loss, we randomly select many pairs of correct example and incorrect ones with the same observed labels, and plot their each epoch’s loss and the cumulative mean loss curves. We consider uniform noise, pairwise noise and structured noise with noise parameter $r=0.4$. As shown in Figure~4 in Section~5.1 of the main paper and the following Figures~\ref{fig:additional_1}-\ref{fig:additional_5}, we can find that correct examples have smaller loss than incorrect ones with the same noisy labels, which further verifies our theoretical explanation. And there exist large fluctuations for single epoch’s loss, which implies that single epoch’s loss may be not reliable enough for sample selection. We can observe that the mean loss is more stable than a single epoch’s loss in experiments. This verifies the effectiveness of using the \emph{mean} loss to select small-loss examples. 

\begin{figure}[h] %
\setlength{\belowcaptionskip}{-0.3cm}   
\centering
\includegraphics[width=\linewidth]{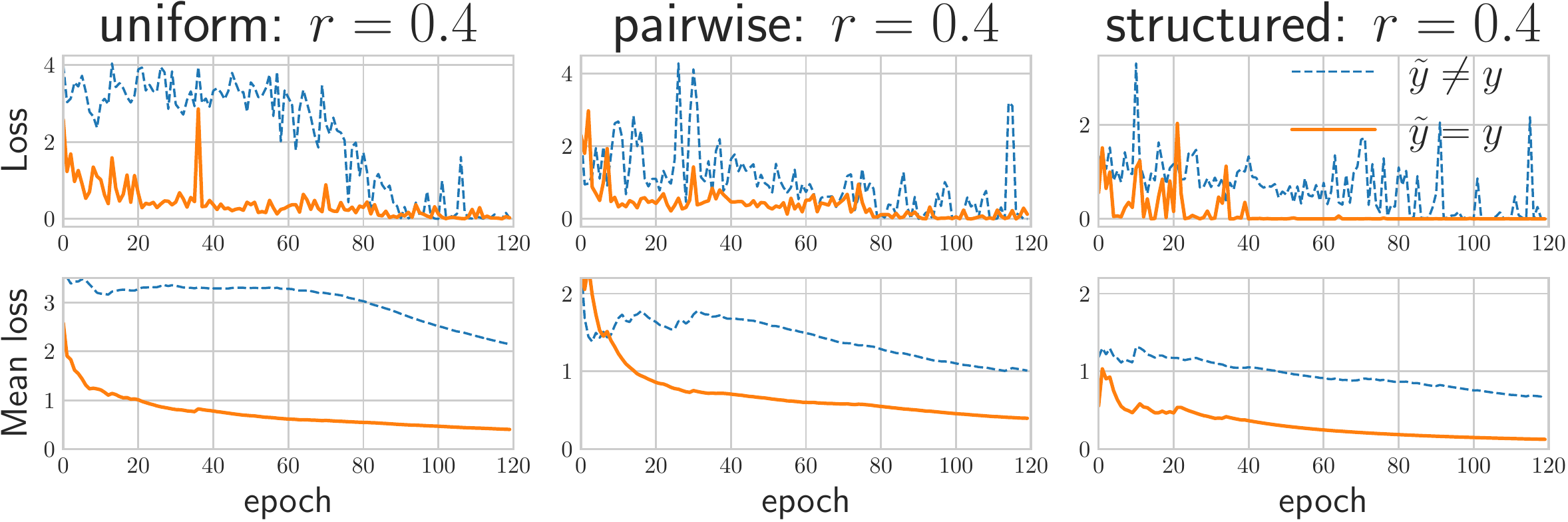} 
\caption{(Additional figure.) Each epoch's loss and the cumulative mean loss of correct example and incorrect example.}
\label{fig:additional_1}
\end{figure}

\begin{figure}[h] %
\setlength{\belowcaptionskip}{-0.3cm}   
\centering
\includegraphics[width=\linewidth]{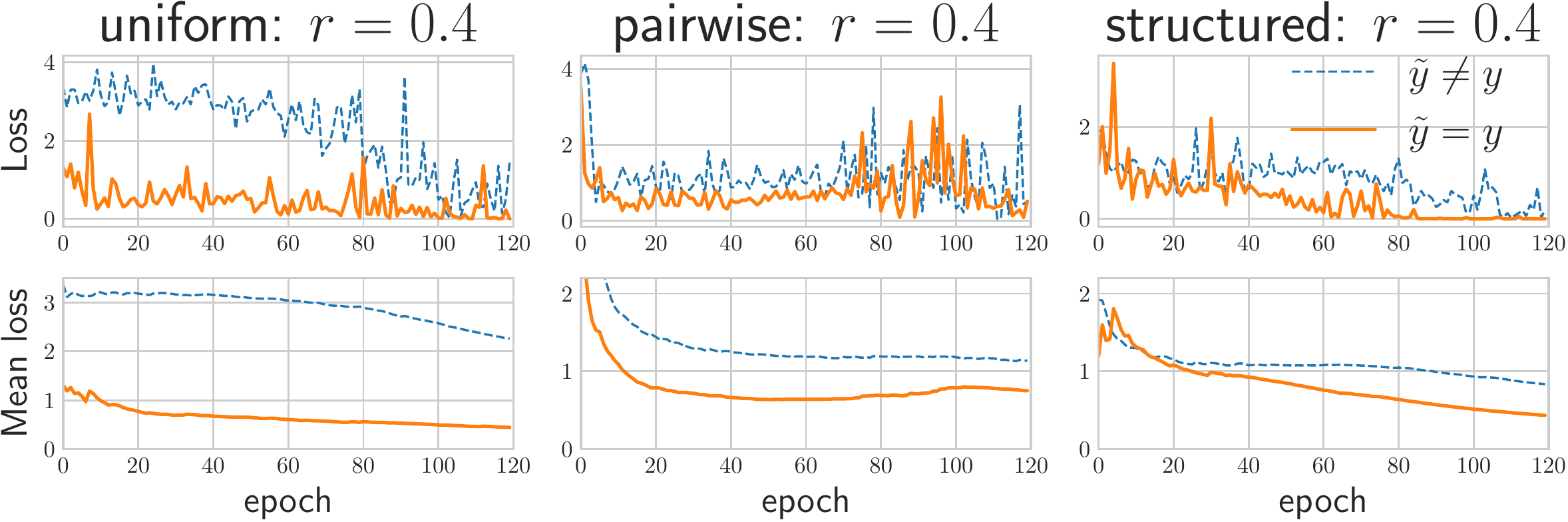}
\caption{(Additional figure.) Each epoch's loss and the cumulative mean loss of correct example and incorrect example.}
\label{fig:additional_2}
\end{figure}

\begin{figure}[h] %
\setlength{\belowcaptionskip}{-0.3cm}  
\centering
\includegraphics[width=\linewidth]{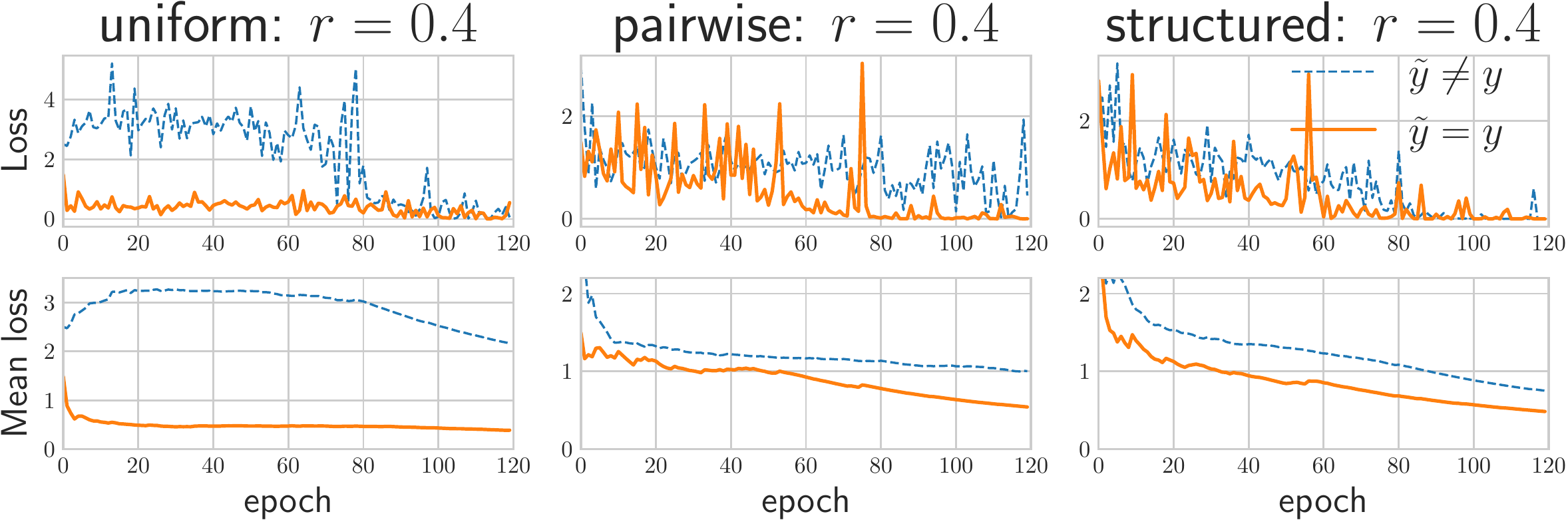} 
\caption{(Additional figure.) Each epoch's loss and the cumulative mean loss of correct example and incorrect example.}
\label{fig:additional_3}
\end{figure}

\begin{figure}[h] %
\centering
\includegraphics[width=\linewidth]{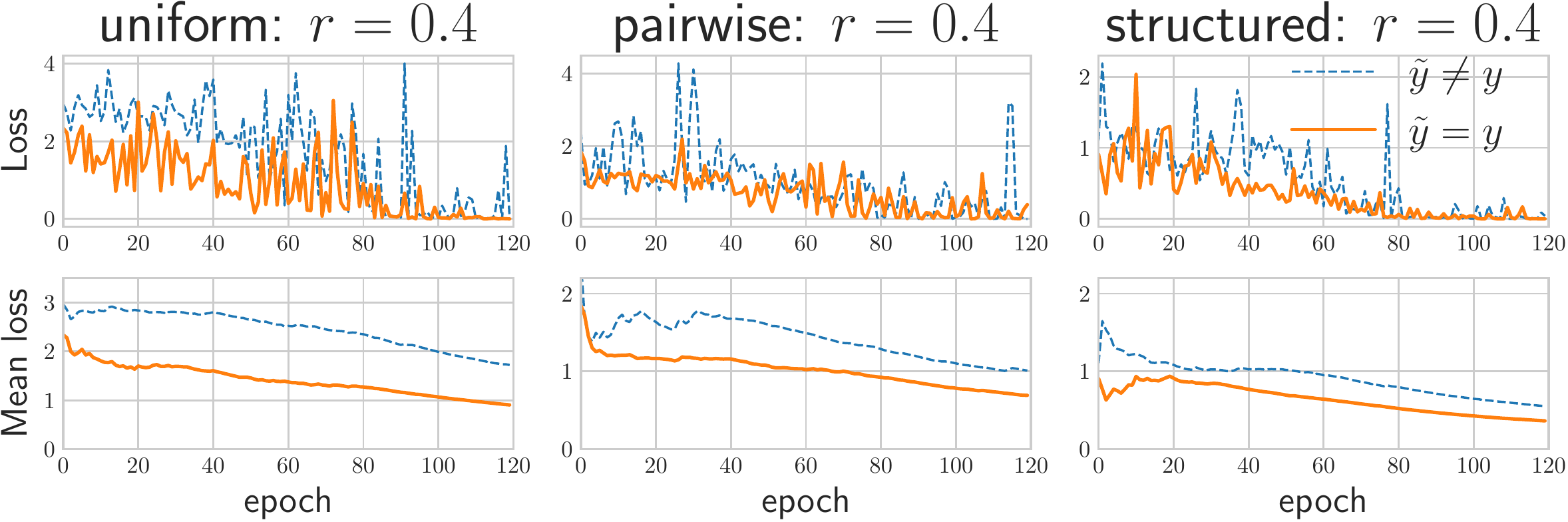}
\caption{(Additional figure.) Each epoch's loss and the cumulative mean loss of correct example and incorrect example.}
\label{fig:additional_4}
\end{figure}

\begin{figure}[h] %
\centering
\includegraphics[width=\linewidth]{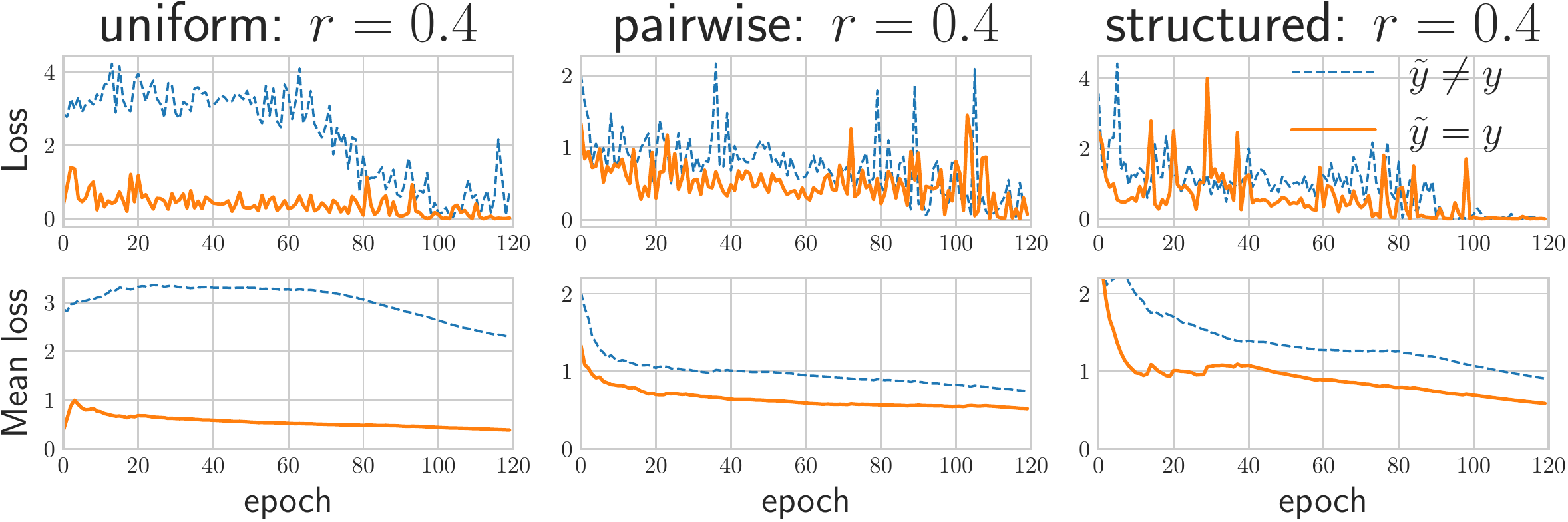} 
\caption{(Additional figure.) Each epoch's loss and the cumulative mean loss of correct example and incorrect example.}
\label{fig:additional_5}
\end{figure}

\newpage
\subsection{Different $g$ with different distances  to $g^*$ }

Theorem~\ref{thm:weak-small-loss} implies when the model $g$ is closer to $g^*$ (\ie, $\epsilon$ is smaller), the loss gap between correct examples and incorrect ones will be bigger and they can be better separated. Now we try to investigate in practice how different $g$ with different distances  to $g^*$ behaves with small-loss criterion.
Since in general, when trained with a bigger training set sampled from the noisy data distribution, the model $g$ will be closer to $g^*$ which minimizes the expected loss on noisy data. Thus we train models using different sizes of the training set to simulate different $g$ with different distances to $g^*$. We plot the model’s training accuracy and test accuracy of each epoch in the training process, the corresponding mean loss values for correct examples and incorrect ones, and the precision of the final selected data (all with the same selection ratio) when using $\frac{1}{4}$, $\frac{1}{2}$, $\frac{3}{4}$ and $1$ of the original noisy training dataset respectively. 
The training dynamics and performances for the uniform label noise ($r=0.4$) is shown in Figure~5 in Section~5.1 of the main paper. 
In Figures~\ref{fig:pairwise_diff_size} and \ref{fig:structured_diff_size}, we show the additional figures for pairwise label noise ($r=0.4$) and structured label noise ($r=0.4$) respectively. It can also be found in these figures that when $g$ is closer to $g^*$, the loss gap between correct examples and incorrect examples is bigger and the precision of the selected data is higher, which verifies that the small-loss criterion will have better utility when the model $g$ is closer to $g^*$.

\begin{figure}[h] %
\setlength{\abovecaptionskip}{-0.005cm}  
\setlength{\belowcaptionskip}{-0.3cm}   
\centering
\includegraphics[width=\linewidth]{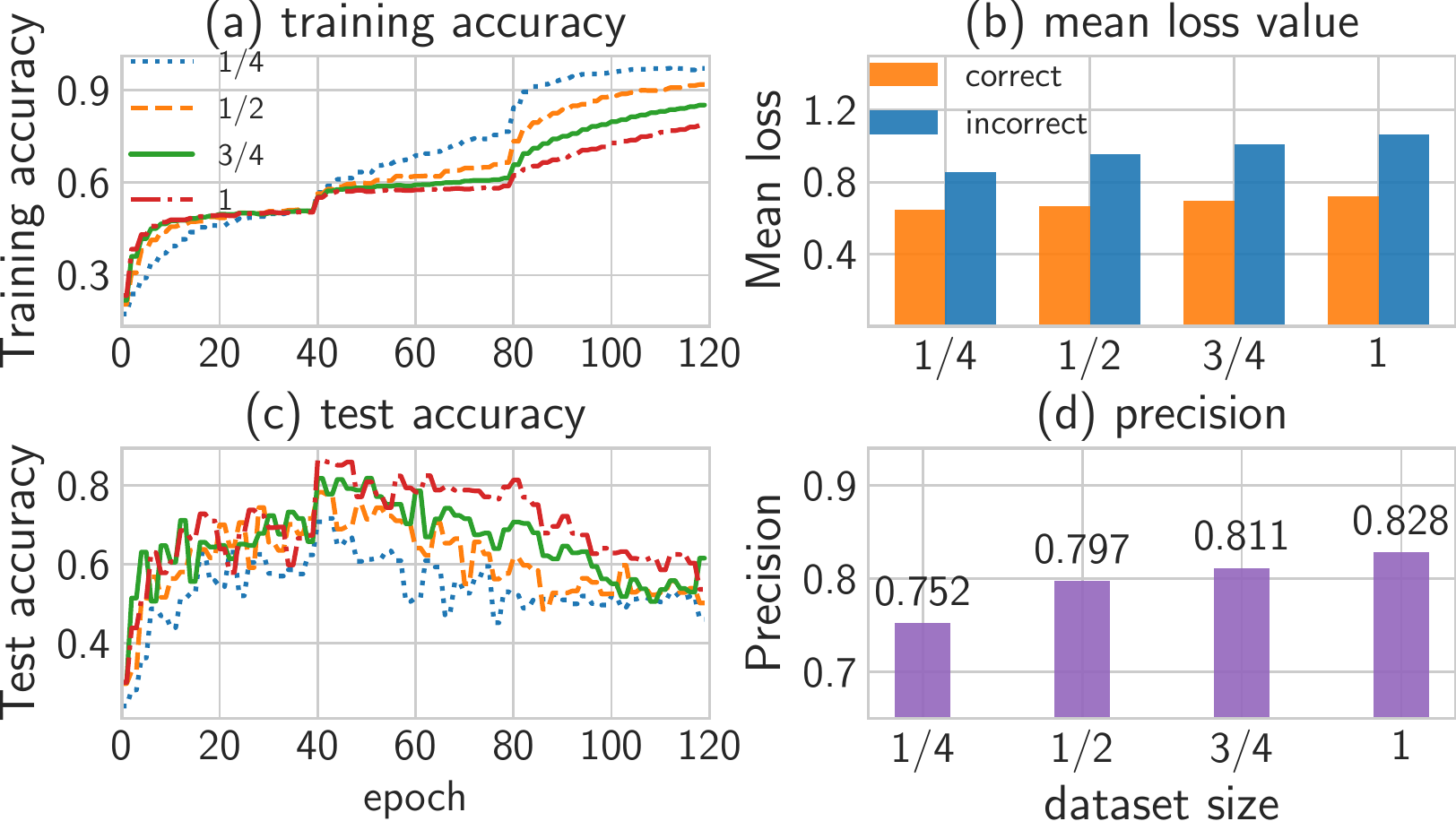}
\caption{Different $g$ with different distances  to $g^*$ (pairwise noise: $r=0.4$).}
\label{fig:pairwise_diff_size}
\end{figure}

\begin{figure}[h] %
\setlength{\belowcaptionskip}{-0.3cm}
\centering
\includegraphics[width=\linewidth]{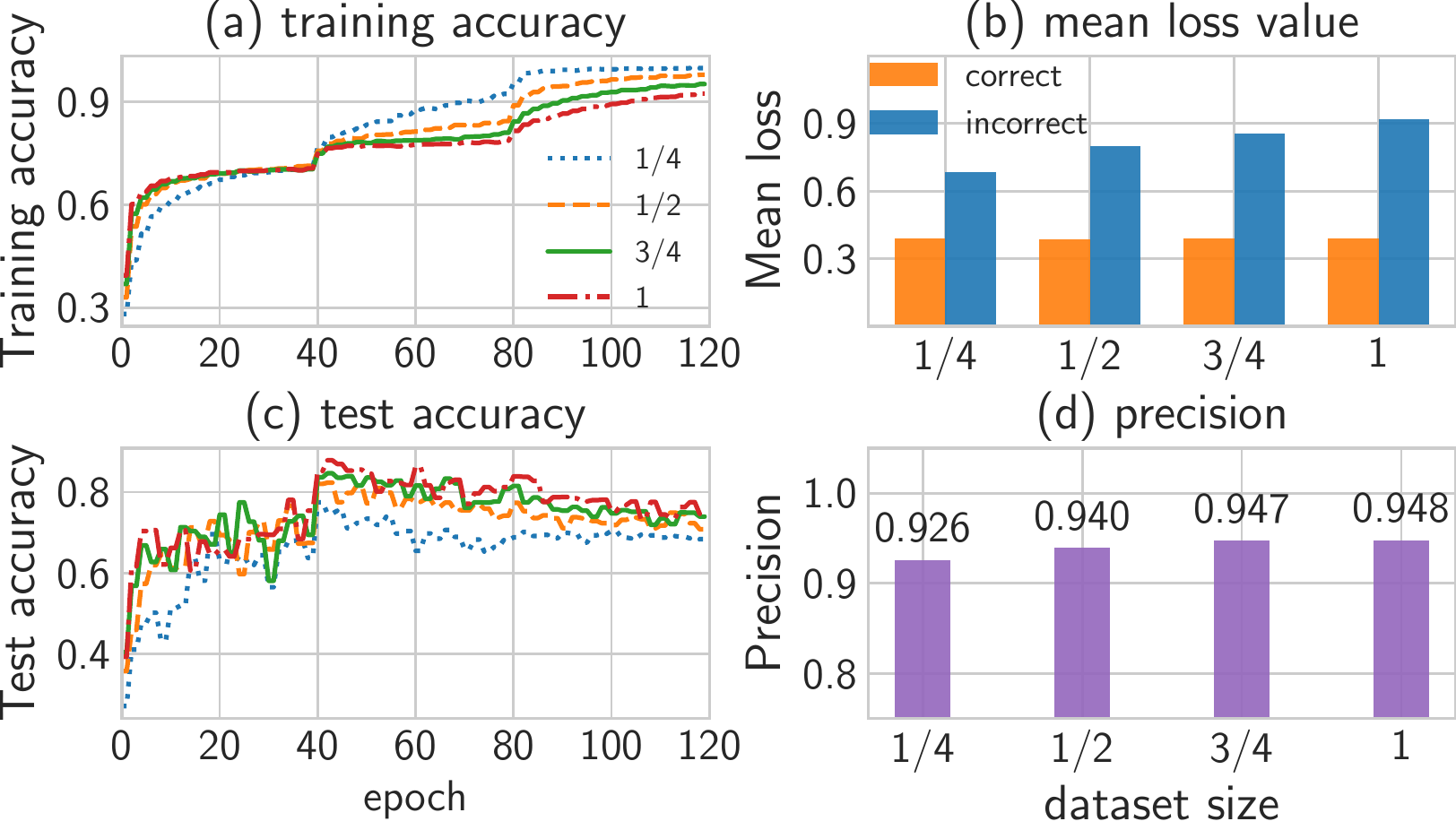}
\caption{Different $g$ with different distances s to $g^*$ (structured noise: $r=0.4$).}
\label{fig:structured_diff_size}
\end{figure}

\section{Evaluations on Benchmark}  
\subsection{Implementation details}
We evaluate the performance of the proposed method on benchmark datasets.
As that in Co-teaching, Co-teaching+ and JoCoR, we also assume that the noise rates are known.
INCV does not need the noise rates, since it embeds a method to estimate the noise rates. $\mathcal{L}_{\text{DMI}}$ and Truncated $\mathcal{L}_q$ are based on new loss functions, and do not need the noise rates.
For fair comparison, all compared methods are carefully implemented with their open-source codes to obey the same setting as ours. 
We also implement ``Cross Entropy'' method, which does not use any special treatments for label noise. The ``best’’ and ``last’’ results are reported for exhibiting the robustness to noisy labels.
Our method is implemented by Pytorch. We use the ``Wide ResNet-28'' model~\cite{oliver2018realistic} for CIFAR-10 and CIFAR-100 following~\cite{berthelot2019mixmatch}. The image data is augmented by horizontal random flip and $32\times32$ random crops after padding with 4 pixels. For the optimizer, we deploy SGD with momentum of $0.9$ and weight decay of $5\times 10^{-4}$. The batch size is set to $128$. We use the default parameters $\beta=0.2$, $\gamma =(\gamma_0 + \gamma_1)/2$ and $\kappa=-\log(0.7)$ in experiments. 
In the first stage, the network is trained on the whole noisy dataset with initial learning rate $0.1$ and multiplied by 0.02 after 60, 120 and 160 epochs (200 in total). After each epoch's training, the loss of each example is calculated. The mean loss is used to rank the examples class by class. For each class, we select a part of data that has minimum mean loss according to the proportion calculated in Alg.~1. For safety and unbiased class distribution, our method selects less data than Co-teaching and JoCoR which set the proportion as $1-\eta$, where $\eta$ is the true noise rate. Then, for RSL, we train the model on the selected data with the same optimization manner as the first stage. While for RSL\_WM, we adopt default hyperparameter setting in \cite{berthelot2019mixmatch} for Weighted\_MixMatch to train the model for fair comparison.

\begin{table*}[t]
\centering
\caption{The relative number of selected data for structured noise CIFAR-10.}
\label{tab:my-table}
\resizebox{\textwidth}{!}{%
\begin{tabular}{c|cccccccccc}
\hline
Relative number & \multicolumn{10}{c}{Structured label noise CIFAR-10 (noise parameter $r$)} \\ \hline
\multicolumn{1}{c|}{class} & 0 & 1 & 2 & 3 & 4 & 5 & 6 & 7 & 8 & 9 \\ \hline
\multicolumn{1}{c|}{$\gamma \ge \frac{1}{1-(1+\beta)r}$} & $1-\beta r$ & $1-\beta r$ & $1-r$ & $1-(1+\beta)r$ & $1-r$ & $1-(1+\beta)r$ & $1$ & $1-\beta r$ & $1$ & $1-r$ \\ \hline
\multicolumn{1}{c|}{$\gamma = 1$} & $1-(1+\beta)r$ & $1-(1+\beta)r$ & $1-(1+\beta)r$ & $1-(1+\beta)r$ & $1-(1+\beta)r$ & $1-(1+\beta)r$ & $1-(1+\beta)r$ & $1-(1+\beta)r$ & $1-(1+\beta)r$ & $1-(1+\beta)r$ \\ \hline
\multicolumn{1}{c|}{$\gamma = (1 +\frac{1}{1-(1+\beta)r})/2$} & $1-\frac{1+\beta}{2}r$ & $1-\frac{1+\beta}{2}r$ & $1-r$ & $1-(1+\beta)r$ & $1-r$ & $1-(1+\beta)r$ & $1-\frac{1+\beta}{2}r$ & $1-\frac{1+\beta}{2}r$ & $1-\frac{1+\beta}{2}r$ & $1-r$ \\ \hline
\end{tabular}%
}
\end{table*}

WebVision dataset has rough $20\%$ noisy labels, and no further information about the noise distribution is known. Following \cite{jiang2018mentornet,chen2019understanding}, we use Inception-ResNet-v2~\cite{szegedy2017inception} as the base model. Since the compared methods in their original papers only report the ``best'’ accuracy, we also only report it. For our method, in the first stage, we train the network for $120$ epochs and use SGD optimizer with an initial learning rate $0.1$, which is divided by $10$ after $40$ and $80$ epochs. We rank all examples by their mean loss class by class and select the first $76\%$ ($\beta=0.2$) examples for each class. Then, for RSL, we train the model on the selected data with the same optimization manner as the first stage. While for RSL\_WM, we use Weighted\_MixMatch to train the final model with all the default parameters. We also compare with F-correction~\cite{patrini2017making}, Decoupling~\cite{malach2017decoupling}, MentorNet~\cite{jiang2018mentornet}, and D2L~\cite{ma2018dimensionality} since they have the same experimental setting as ours on WebVision dataset.

\subsection{The Hyperparameters $\beta$ and $\gamma$}\label{sup:eq4-5}

For the $i$-th class with noise rate $\eta_i$, it is reasonable that the selected proportion $\textit{prop}(i)$ is a little less than $1-\eta_i$. In Section~4 of the main paper, we firstly propose to set $\textit{prop}(i) = \max \{1-(1+\beta)\eta_i,\,\, (1-\beta)(1-\eta_i)\}$, where $0\le\beta\le1$ is a parameter which can be set as $0.2$, and bigger $\beta$ implies that we select less data for safety. In here, we further explain it and investigate the effect of $\beta$ on the selected data. Actually, the above equation is equivalent that if $\eta_i\le0.5$, we set $\textit{prop}(i) =1-(1+\beta)\eta_i$; while if $\eta_i>0.5$, we set $\textit{prop}(i) =(1-\beta)(1-\eta_i)$. Figure~\ref{fig:beta} depicts the relationship between $\textit{prop}(i)$ and $\eta_i$ when the hyperparameter $\beta\in\{0$, $0.1$, $0.2$, $0.3\}$. 
\begin{figure}[H] %
\setlength{\abovecaptionskip}{-0.0cm}  
\setlength{\belowcaptionskip}{-0.3cm}   
\centering
\includegraphics[width=0.5\linewidth]{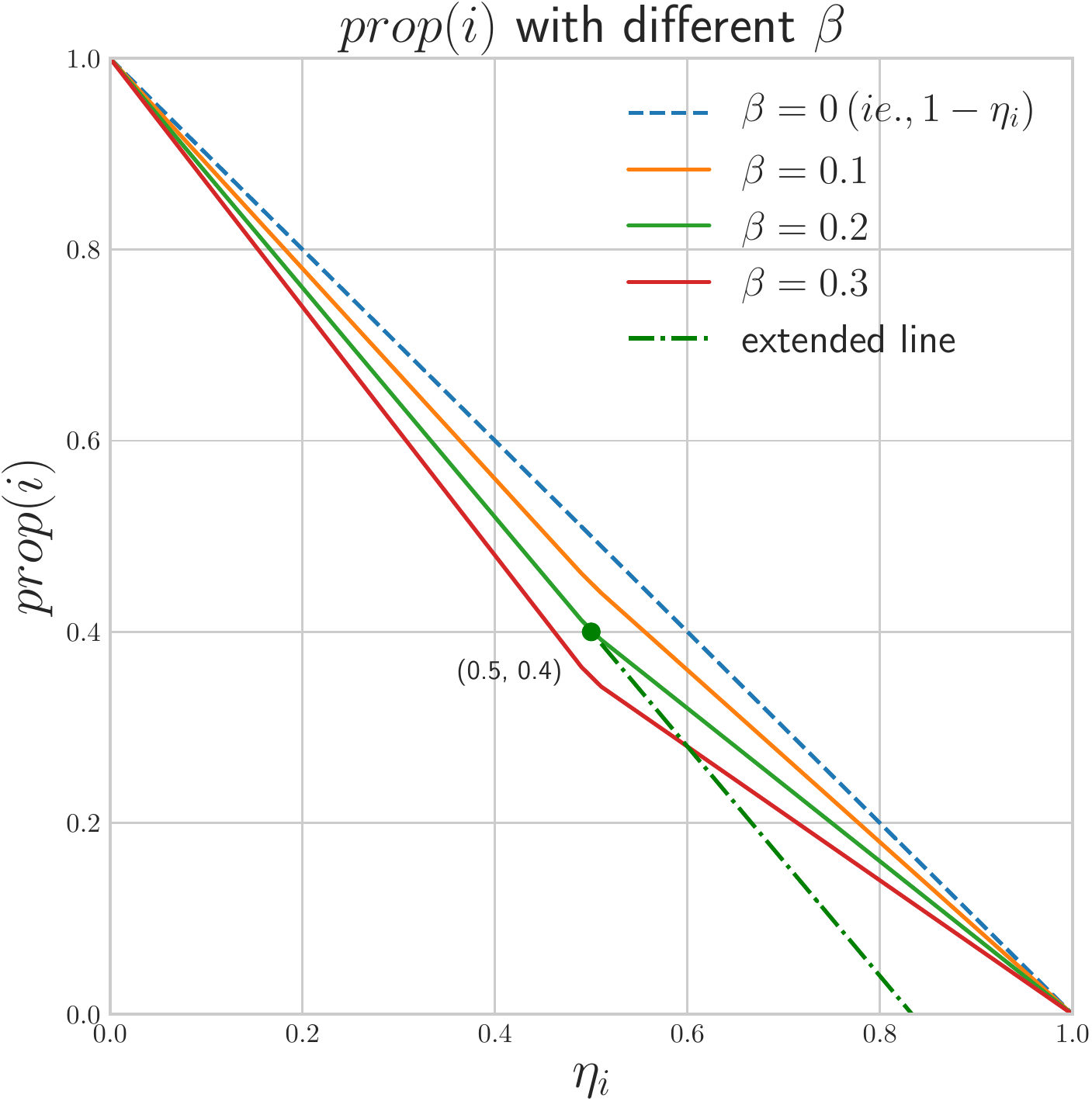}
\caption{The schemes of $\textit{prop}(i)$ with different $\beta$.}
\label{fig:beta}
\end{figure}
In Section~4 of the main paper, we finally use Eq.~(2)
to further adjust the selected number for each class to alleviate the class distribution shift, \ie, $\textit{num}(i) = \min \{\gamma\cdot p_i\times m, \,\, \textit{prop}(i)\times n_i\}$. Denote $\gamma_0 =1$ and $\gamma_1=\max_{1\le i\le c}\{\frac{\textit{prop}(i)\cdot n_i}{p_i\cdot m}\}$, if $\gamma = \gamma_0$, then $\textit{num}(i) = p_i\times m$ which exactly matches the distribution $[p_1, \dots, p_c]$ but may waste many useful data;
if $\gamma\ge\gamma_1$, then $\textit{num}(i)$ collapses to $\textit{prop}(i)\times n_i$. In practice, setting $\gamma = (\gamma_0+\gamma_1)/2$ may be a reasonable choice.

\begin{figure*}[t] %
\centering
\subfigure[uniform label noise]{
\begin{minipage}[t]{0.318\linewidth}
\centering
\includegraphics[width=\linewidth]{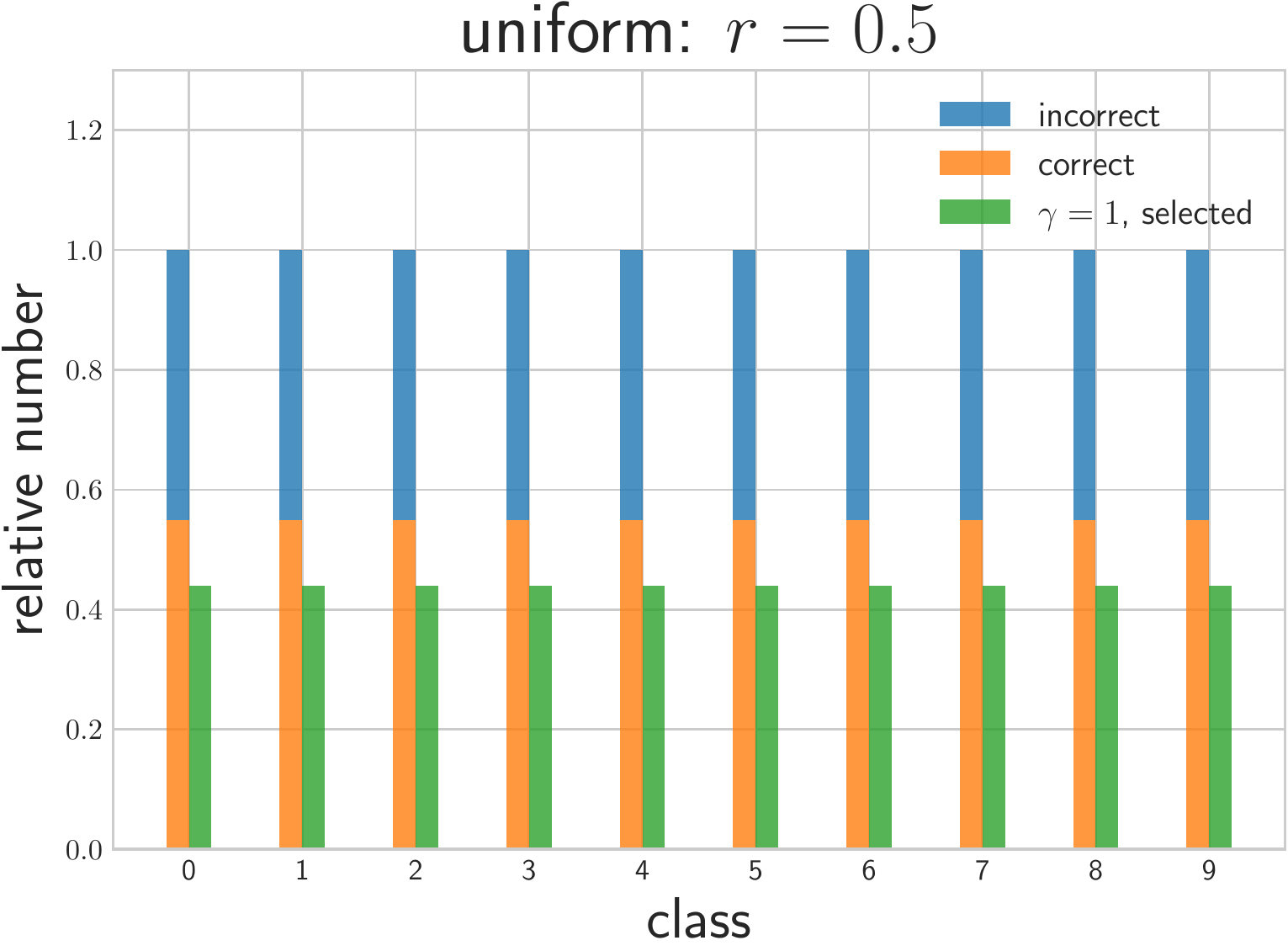}
\label{fig:UN-num}
\end{minipage}%
}
\subfigure[pairwise label noise ]{
\begin{minipage}[t]{0.318\linewidth}
\centering
\includegraphics[width=\linewidth]{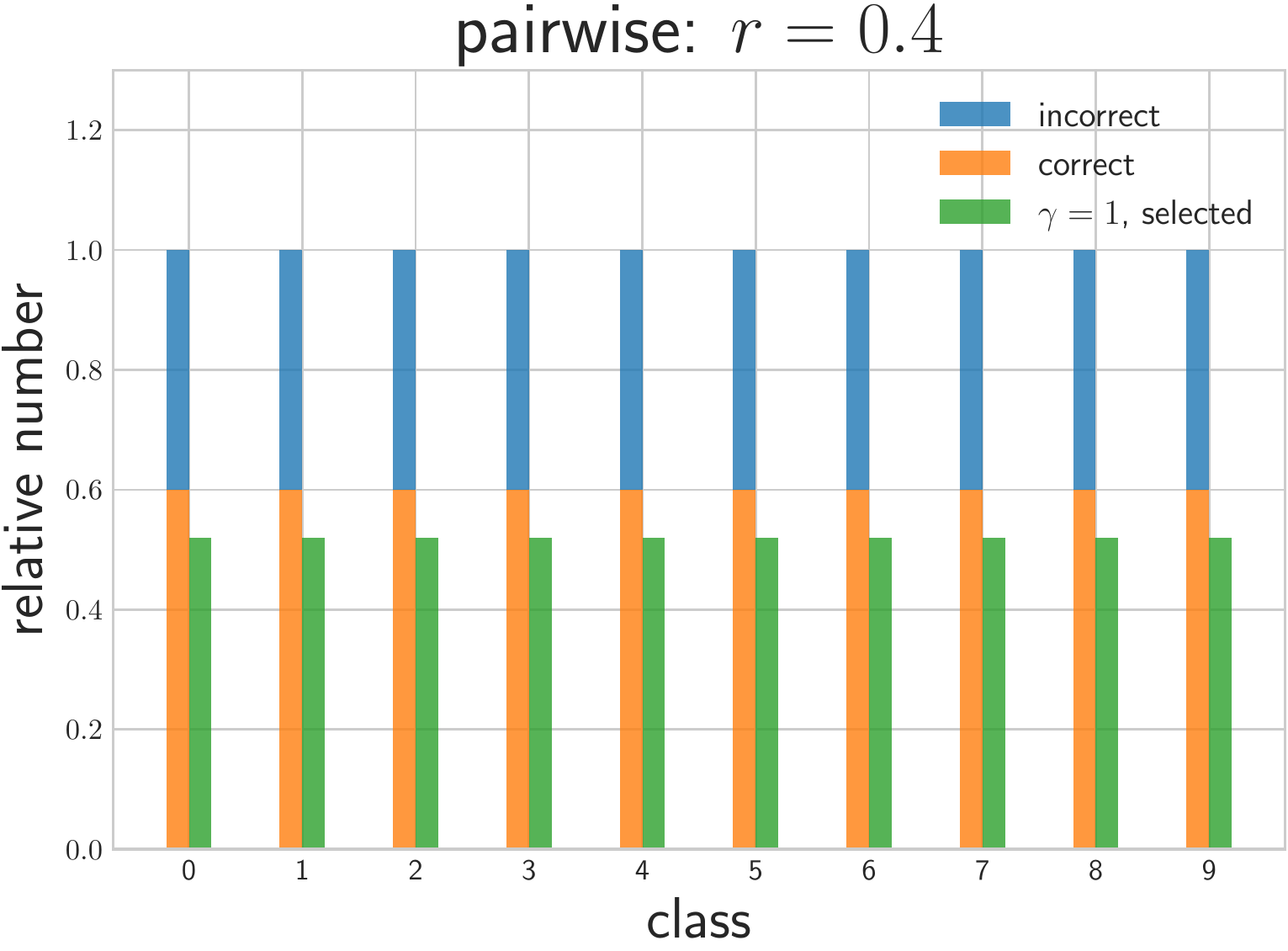} 
\label{fig:PN-num}
\end{minipage}
}
\subfigure[structured label noise ]{
\begin{minipage}[t]{0.318\linewidth}
\centering
\includegraphics[width=\linewidth]{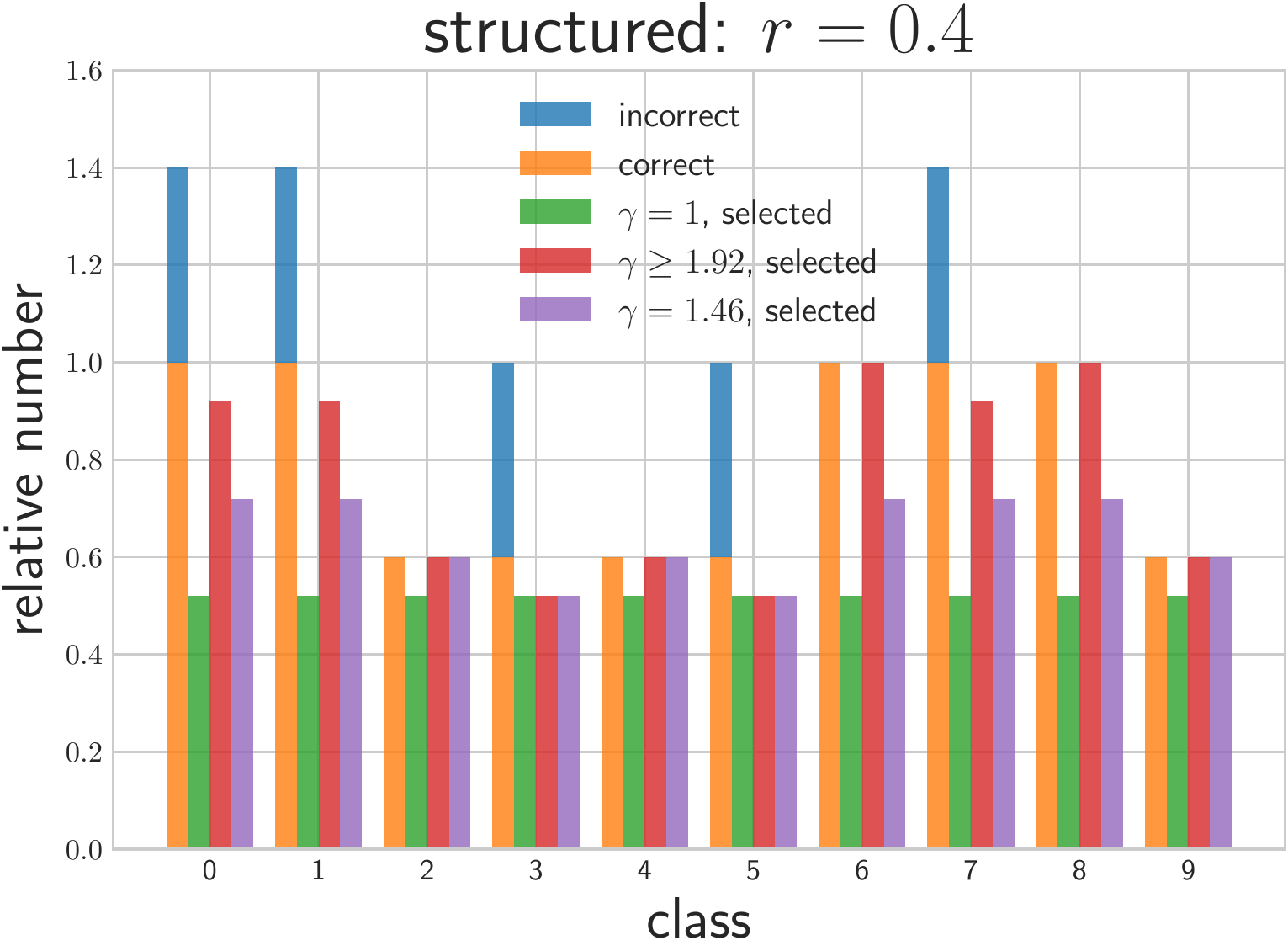} 
\label{fig:AN-num}
\end{minipage}
}
\caption{The selected relative number on noisy CIFAR-10 with $\beta =0.2$ for different $\gamma$.}
\label{fig:relative-num}
\end{figure*}

For uniform and pairwise label noise with uniform true class distribution (\ie, $p_i=p_j$, $\forall i,j$), we actually have $\gamma_1=\max_{1\le i\le c}\{\frac{\textit{prop}(i)\cdot n_i}{p_i\cdot m} \}=1=\gamma_0$, and the selected data by using $\textit{prop}(i)$ (equivalently, setting $\gamma=1$ for $num(i)$) already matches the true class distribution, \ie, $num(i)$ collapses to $\textit{prop}(i)\times n_i$. But for structured label noise, directly using $\textit{prop}(i)\times n_i$ will cause class distribution shift, which may cause serious class imbalance and influence the learning process. Taking structured label noise CIFAR-10 (Figure~1 (c)) for example, directly using $\textit{prop}(i)\times n_i$ 
 (equivalently, setting $\gamma \ge \gamma_1=\frac{1}{1-(1+\beta)r}$) will make the selected data deviates from the true class distribution. Setting $\gamma=\gamma_0=1$ will get unbiased selected data, but may waste many useful data. There should be a trade-off between the number of  selected data and the unbiasedness of the class distribution.
Table~\ref{tab:my-table} shows the relative number of the selected data for each class with different $\gamma$, where the relative number is defined as $\frac{\textit{num}(i)}{p_i\times(\sum_{j=1}^c n_j)}$.  To better present the effect of $\gamma$ on the selected data, we plot the relative number of the selected data for uniform noise ($r=0.5$), pairwise noise ($r=0.4$) and structured noise ($r=0.4$) with the default $\beta=0.2$ in Figure~\ref{fig:relative-num}.

The selected proportion is a hyperparameter in our method, which depends on noise rates. Although most methods based on the sample selection strategy~\cite{han2018co,yu2019does,wei2020combating} or loss correction~\cite{natarajan2013learning,patrini2017making} assume that the exact noise information is known, we may not know this in practice. Usually, we resort to some methods~\cite{patrini2017making,hendrycks2018using} or manually verify a small group of randomly selected samples~\cite{xiao2015learning} to estimate the noise information. 
Let $\hat{\eta}_i$ denote the estimation of $\eta_i$, there are two cases: If $\hat{\eta}_i > \eta_i$, we may select fewer but cleaner examples from noisy data; If $\hat{\eta}_i < \eta_i$,  by setting the selected proportion less than $1-\hat{\eta}_i$, we greatly diminish the damage of selecting incorrect examples. 

\begin{table*}[t]
\centering
\caption{The accuracy with different $\beta$ on CIFAR-10 with the default $\gamma$ and $\kappa$.}
\label{tab:beta-CIFAR10}
\resizebox{0.8\textwidth}{!}{%
\begin{tabular}{c|c|cccc|cccc|cccc}
\hline
\multicolumn{2}{c|}{CIFAR-10} & \multicolumn{4}{c|}{uniform noise ($r=0.5$)} & \multicolumn{4}{c|}{pairwise noise ($r=0.4$)} & \multicolumn{4}{c}{structured noise ($r=0.4)$} \\
\hline
\multicolumn{2}{c|}{$\beta$} & $0$ & $0.1$ & $0.2$ & $0.3$ & $0$ & $0.1$ & $0.2$ & $0.3$ & $0$ & $0.1$ & $0.2$ & $0.3$ \\
\hline
\multirow{2}{*}{RSL\_WM} & best & 92.75 & 93.24 & \textbf{93.38} & 93.34 & 85.23 & 87.66 & 89.27 & \textbf{90.77} & 90.83 & 91.01 & \textbf{91.17} &  90.40\\
 & last & 92.45  & 93.18 & \textbf{93.27} & 93.26 & 84.19 & 87.04 & 88.85 & \textbf{90.55} & 90.41 &\textbf{90.72}  & 90.63 & 90.02 \\
 \hline
\end{tabular}%
}
\end{table*}
\begin{table*}[b]
\centering
\caption{The accuracy with different $\beta$ on CIFAR-100 with the default $\gamma$ and $\kappa$.}
\label{tab:beta-CIFAR100}
\resizebox{0.65\textwidth}{!}{%
\begin{tabular}{c|c|cccc|cccc}
\hline
\multicolumn{2}{c|}{CIFAR-100} & \multicolumn{4}{c|}{uniform noise ($r=0.4$)} & \multicolumn{4}{c}{pairwise noise ($r=0.4$)} \\
\hline
\multicolumn{2}{c|}{$\beta$} & $0.0$ & $0.1$ & $0.2$ & $0.3$ & $0.0$ & $0.1$ & $0.2$ & $0.3$ \\
\hline
\multirow{2}{*}{RSL\_WM} & best & 71.20 & 71.28 & \textbf{71.51} & 70.86 & 51.96& 52.14& \textbf{53.25}  & 52.64\\
 & last & 70.70  & \textbf{70.84} & 70.69 & 70.71 & 50.54 & 51.27& \textbf{52.44}  & 52.24  \\
 \hline
\end{tabular}%
}
\end{table*}

\begin{table*}[b]
\centering
\caption{The accuracy with different $\gamma$ on structured noise CIFAR-10 with the default $\beta$ and $\kappa$.}
\label{tab:gamma}
\resizebox{0.75\textwidth}{!}{%
\begin{tabular}{c|c|ccc|ccc|ccc}
\hline
\multicolumn{2}{c|}{CIFAR-10} & \multicolumn{3}{c|}{structured noise ($r=0.2$)} & \multicolumn{3}{c|}{structured noise ($r=0.3$)} & \multicolumn{3}{c}{structured noise ($r=0.4)$} \\
\hline
\multicolumn{2}{c|}{$\gamma$} & 1.0 & $\frac{22}{19}$ & $\frac{25}{19}$ & 1.0 & $\frac{41}{32}$ & $\frac{25}{16}$ &  1.0 & $\frac{19}{13}$& $\frac{25}{13}$ \\
\hline
\multirow{2}{*}{RSL\_WM} & best & \textbf{93.15} & 93.12 & 92.04 & 92.74 & \textbf{92.78} & 91.37 & 90.79 &  \textbf{91.17} &  89.22 \\
 & last & \textbf{92.97} & 92.83 & 91.83 & 92.30 & \textbf{92.34} & 89.19 & 90.34 & \textbf{90.63} & 85.69 \\
 \hline
\end{tabular}%
}
\end{table*}

In practice, the exact values of noise rates are not known, inaccurate estimation may hurt the performance. Our method uses the parameters $\beta$ and $\gamma$ as slack variables to relieve the influence of the inaccurate estimation. Thus in our method, the study of the sensitivity of the performance to the estimated noise rate $\hat{\eta}_i$ can be converted into the study of the sensitivity of the performance to the hyperparameters $\beta$ and $\gamma$, since $\beta$ and $\gamma$ are two hyperparameters controlling the selected proportion. We consider $\beta \in \{0,0.1,0.2,0.3\}$ and $\gamma \in\{\gamma_0, (\gamma_0+\gamma_1)/2, \gamma_1\}$, where $\gamma_0 = 1$ and $\gamma_1=\max_{1\le i\le c}\{\frac{\textit{prop}(i)\cdot n_i}{p_i\cdot m}\}$. When analyzing one of the hyperparameters $\beta$, $\gamma$ and $\kappa$, we fix the other two as the default values. The default values for $\beta$, $\gamma$, and $\kappa$ are $0.2$, $(\gamma_0 +\gamma_1)/2$ and $-\log(0.7)$ respectively. 

Tables~\ref{tab:beta-CIFAR10} and~\ref{tab:beta-CIFAR100} show that in general, selecting less but cleaner data ($\beta>0$) can achieve better performance, and larger $\beta$ is needed when separating correct examples from incorrect ones is difficult (\eg, pairwise noise $r=0.4$ for CIFAR-10). But when $\beta$ is 
too large, many useful labels are wasted and the performance may decrease. In practice, $\beta=0.2$ may be a good choice.

Table~\ref{tab:gamma} shows that better performance can be achieved by considering the class distribution shift. When the data is relatively abundant (\eg, structured noise $r=0.2$), setting $\gamma=\gamma_0$ achieves better performance than setting $\gamma=(\gamma_0+\gamma_1)/2$.

\begin{table*}[t]
\centering
\caption{The accuracy with different $\kappa$ on CIFAR-10 with the default $\beta$ and $\gamma$.}
\label{tab:kappa-CIFAR10}
\resizebox{\textwidth}{!}{%
\begin{tabular}{c|c|cccc|cccc|cccc}
\hline
\multicolumn{2}{c|}{CIFAR-10} & \multicolumn{4}{c|}{uniform noise ($r=0.5$)} & \multicolumn{4}{c|}{pairwise noise ($r=0.4$)} & \multicolumn{4}{c}{structured noise ($r=0.4)$} \\
\hline
\multicolumn{2}{c|}{$\kappa$} & $0.0$ & $-\log(0.9) $ & $-\log(0.7)$ & $-\log(0.5)$ & $0.0$ & $-\log(0.9)$ & $-\log(0.7)$ & $-\log(0.5)$ & $0.0$ & $-\log(0.9)$ & $-\log(0.7)$ & $-\log(0.5)$ \\
\hline
\multirow{2}{*}{RSL\_WM} & best & 93.21 & 93.29 & \textbf{93.38} & 93.18 & 88.87 & 89.48 & 89.27 & \textbf{89.77} &  90.27& 90.79 & \textbf{91.17} & 90.25 \\
 & last & 92.95 & 93.17 & \textbf{93.27} & 93.16 & 88.52 & 89.13 & 88.85 & \textbf{89.30} & 90.03 &90.48  & \textbf{90.63} & 89.63 \\
 \hline
\end{tabular}%
}
\end{table*}

\subsection{The Hyperparameter $\kappa$}\label{sup:kappa}
Now we first briefly recap the standard MixMatch process~\cite{berthelot2019mixmatch} which is omitted in Section~4 of the main paper due to space limit. Given a batch of labeled data $\mathcal{L}=\{(\bm{x}_b, p_b)\}_{1\le b\le B}$ and a batch of unlabeled data $\mathcal U=\{\bm{u}_b\}_{1\le b\le B}$, MixMatch produces a batch of augmented labeled examples $\mathcal{L}'$ and a batch of augmented unlabeled examples $\mathcal{U}'$ with ``guessing'' labels by unifying data augmentation, ``label guessing'', ``sharpening'' and slightly modified MixUp~\cite{zhang2017mixup}. This process can be formulated as $(\mathcal{L}', \mathcal{U}') = \text{MixMatch}(\mathcal{L}, \mathcal{U}, T, K, \alpha)$, where hyperparameter $K$ represents the number of augmentations, $T$ represents the sharpening temperature and $\alpha$ represents the Beta distribution parameter for MixUp. Then, a standard cross-entropy loss is used on $\mathcal{L}'$, while the loss on $\mathcal{U}'$ is calculated with mean square error loss which is less sensitive to incorrect labels than cross-entropy loss, \ie, $L_\mathcal{L} = \frac{1}{|\mathcal{L}'|} \sum_{(\bm x, p)\in \mathcal{L}'} H(p, p_{\text{model}}(\hat y | \bm x))$ and $L_\mathcal{U} = \frac{1}{c|\mathcal{U}'|}\sum_{(\bm u, q)\in \mathcal{U}'}\|q-p_{\text{model}}(\hat y | \bm u)\|_2^2$. The overall loss $L$ is calculated as $L = L_\mathcal{L} +\lambda_{\mathcal U}L_\mathcal{U}$ and the final classifier is learned by optimizing $L$, where $\lambda_{\mathcal{U}}$ is the weight parameter for unsupervised loss term.

The selected samples with the ``small-loss'' criterion may still have noise, and thus reweighing examples is needed to relieve the damage of incorrect examples. 
Generally, large loss implies that the probability of the example having an incorrect label is relatively high, so we reweigh the selected examples to alleviate the influence of label noise. 
The weight for an example $(\bm{x}, \tilde{y})$ is calculated as:
\[w(\bm x, \tilde{y}) = \exp\big(-\kappa\frac{\bar{\ell}(\bm x,\tilde{y}) -\ell_*(i)}{\ell^*(i) - \ell_*(i)}\big),\]
where $\kappa\ge 0$ is a hyperparameter. Let $t \triangleq \frac{\bar{\ell}(\bm x,\tilde{y}) -\ell_*(i)}{\ell^*(i) - \ell_*(i)} \in [0,1]$, the weight $w = \exp(-\kappa t)$ is shown in Figure~\ref{fig:kappa} for $\kappa \in \{0.0$, $-\log(0.9)$, $-\log(0.7)$, $-\log(0.5)$, $-\log(0.3)$, $-\log(0.1)\}$. 
When setting $\kappa=0.0$,  it means that we think all the selected examples are reliable and assign equal weights for all examples. When setting $\kappa>0$, the examples with larger mean loss will be assigned smaller weights. We embed the weight into MixMatch with weighted resampling technique and call it Weighted\_MixMatch. Denote $D_u = \{\bm x|\,\forall\, (\bm x,\tilde{y})\in \tilde{D}\backslash D_\text{sel}\}$, the above process can be formulated as $D_\text{sel\_WM} = \text{Weighted\_MixMatch}(D_\text{sel}, D_u)$. Then we name the method of training $g(\bm x;\Theta)$ with $D_\text{sel\_WM}$ by Weight\_MixMatch rather than $D_\text{sel}$ as RSL\_WM. 
We adopt the default hyper-parameters of standard MixMatch and additionally analyze the influence of $\kappa$ in experiments.

\begin{figure}[H] 
\centering
\includegraphics[width=0.6\linewidth]{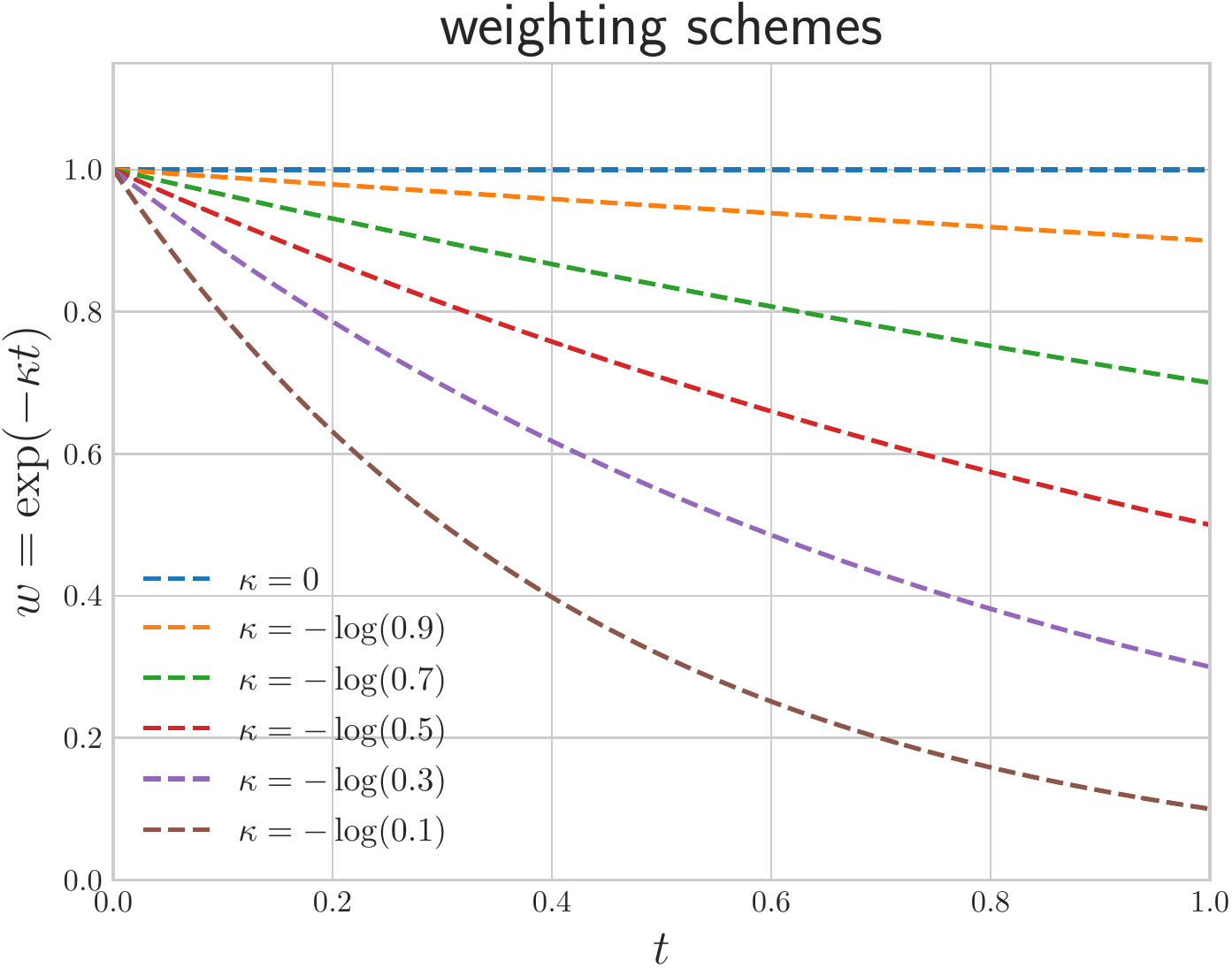}
\caption{Weighting schemes with different $\kappa$ values.}
\label{fig:kappa}
\end{figure}

Tables~\ref{tab:kappa-CIFAR10} and~\ref{tab:kappa-CIFAR100} show that better performance can be achieved by reweighing the selected examples ($\kappa>0$), and larger $\kappa$ is needed when the precision of the selected data is relatively low (\eg, pairwise noise $r=0.4$ for CIFAR-10 and CIFAR-100). In practice, $\kappa=-\log(0.7)$ may be a good choice.

\begin{table*}[t] 
\centering 
\caption{The accuracy with different $\kappa$ on CIFAR-100 with the default $\beta$ and $\gamma$.}
\label{tab:kappa-CIFAR100}
\resizebox{0.7\textwidth}{!}{%
\begin{tabular}{c|c|cccc|cccc}
\hline
\multicolumn{2}{c|}{CIFAR-100} & \multicolumn{4}{c|}{uniform noise ($r=0.4$)} & \multicolumn{4}{c}{pairwise noise ($r=0.4$)} \\
\hline
\multicolumn{2}{c|}{$\kappa$} & $0.0$ & $-\log(0.9) $ & $-\log(0.7)$ & $-\log(0.5)$ & $0.0$ & $-\log(0.9)$ & $-\log(0.7)$ & $-\log(0.5)$ \\
\hline
\multirow{2}{*}{RSL\_WM} & best & 71.33 & 71.47 & \textbf{71.51} & 71.48 & 52.47 & 53.05 & 53.25 & \textbf{53.26} \\
 & last & 70.28 & 70.63 & 70.69 & \textbf{71.01} & 52.16 & 52.13 & 52.44 & \textbf{52.81} \\
 \hline
\end{tabular}%
}
\end{table*}

\end{appendices}

\newpage
\mbox{}
\newpage

\small 
\bibliographystyle{named}
\bibliography{my_egbib}

\end{document}